\documentclass{article}
\usepackage{amsmath, amsthm, amssymb,float}
\usepackage{amsfonts}
\usepackage[linesnumbered,ruled]{algorithm2e}

\usepackage{graphicx}
\usepackage{caption}
\usepackage{subcaption}
\usepackage{centernot}

\usepackage{textcomp}
\usepackage{cases}
\usepackage{braket}
\usepackage[normalem]{ulem}
\usepackage{tabularx,diagbox}

\usepackage{chngcntr}

\usepackage{xspace}

\usepackage{pifont}

\usepackage[margin=2.5cm]{geometry}

\usepackage[colorlinks=true, urlcolor=blue, anchorcolor=blue, citecolor=blue,filecolor=blue,linkcolor=blue,menucolor=blue,pagecolor=blue]{hyperref}

\usepackage[capitalise]{cleveref}
\crefformat{equation}{(#2#1#3)}

\def\red{\textcolor{red}}

\theoremstyle{theorem}
\newtheorem{theorem}{Theorem}
\numberwithin{theorem}{section}

\theoremstyle{lemma}
\newtheorem{lemma}[theorem]{Lemma}

\theoremstyle{remark}
\newtheorem{remark}[theorem]{Remark}

\theoremstyle{corollary}

\theoremstyle{proposition}
\newtheorem{proposition}[theorem]{Proposition}

\theoremstyle{definition}
\newtheorem{definition}[theorem]{Definition}

\theoremstyle{claim}

\DeclareMathOperator{\Tr}{Tr}
\newcommand{\GNMR}{{\texttt{GNMR}}\xspace}
\newcommand{\GNIMC}{{\texttt{GNIMC}}\xspace}
\DeclareMathOperator*{\argmax}{arg\,max}
\DeclareMathOperator*{\argmin}{arg\,min}
\newenvironment{psmallmatrix}
  {\bigl(\begin{smallmatrix}}
  {\end{smallmatrix}\bigr)}

\begin{document}

\title{Inductive Matrix Completion: \\No Bad Local Minima and a Fast Algorithm}	

\author{Pini Zilber\footnotemark[1]\thanks{Faculty of Mathematics and Computer Science, Weizmann Institute of Science \newline (\href{mailto:pini.zilber@weizmann.ac.il}{pini.zilber@weizmann.ac.il}, \href{mailto:boaz.nadler@weizmann.ac.il}{boaz.nadler@weizmann.ac.il})}
\and Boaz Nadler\footnotemark[1]}
\date{}

\maketitle
\nopagebreak

\begin{abstract}
The inductive matrix completion (IMC) problem is to recover a low rank matrix from few observed entries while incorporating prior knowledge about its row and column subspaces. 
In this work, we make three contributions to the IMC problem: (i) we prove that under suitable conditions, the IMC optimization landscape has no bad local minima;
(ii) we derive a simple scheme with theoretical guarantees to estimate the rank of the unknown matrix;
and (iii) we propose \GNIMC, a simple Gauss-Newton based method to solve the IMC problem, analyze its runtime and derive recovery guarantees for it.
The guarantees for \GNIMC are sharper in several aspects than those available for other methods, including a quadratic convergence rate, fewer required observed entries and stability to errors or deviations from low-rank. Empirically, given entries observed uniformly at random, \GNIMC recovers the underlying matrix substantially faster than several competing methods.
\end{abstract}

\section{Introduction}
In low rank matrix completion, a well-known problem that appears in various applications, the task is to recover a rank-$r$ matrix $X^*\in \mathbb R^{n_1\times n_2}$ given few of its entries, where $r \ll \min\{n_1,n_2\}$.
In the problem of inductive matrix completion (IMC), beyond being low rank, $X^*$ is assumed to have additional structure as follows: its columns belong to the range of a known matrix $A\in \mathbb R^{n_1\times d_1}$ and its rows belong to the range of a known matrix $B\in \mathbb R^{n_2\times d_2}$, where $r \leq d_1 \leq n_1$ and $r \leq d_2 \leq n_2$. Hence, $X^*$ may be written as $X^* = AM^*B^\top$, and the task reduces to finding the smaller matrix $M^*\in \mathbb R^{d_1\times d_2}$.
In practice, the low rank and/or the additional structure assumptions may hold only approximately, and in addition, the observed entries may be corrupted by noise.

The side information matrices $A,B$ may be viewed as feature representations. For example, in movies recommender systems, the task is to complete a matrix $X^*$ of the ratings given by $n_1$ users to $n_2$ movies.
The columns of $A, B$ may correspond to viewers' demographic details (age, gender) and movies' properties (length, genre), respectively \cite{abernethy2009new,menon2011response,chen2012svdfeature,yao2019collaborative}.
The underlying assumption in IMC is that uncovering the relations between the viewers and the movies in the feature space, as encoded in $M^*$, suffices to deduce the ratings $X^* = AM^*B^\top$.
Other examples of IMC include multi-label learning \cite{xu2013speedup,si2016goal,zhang2018fast}, disease prediction from gene/miRNA/lncRNA data \cite{natarajan2014inductive,chen2018predicting,lu2018prediction} and link prediction in networks \cite{menon2011link,chiang2018using}.

If the side information matrices allow for a significant dimensionality reduction, namely $d \ll n$ where $d = \max\{d_1,d_2\}$ and $n = \max\{n_1,n_2\}$, recovering $X^*$ is easier from both theoretical and computational perspectives.
From the information limit aspect, the minimal number of observed entries required to complete a matrix of rank $r$ with side information scales as $\mathcal O(d r)$, compared to $\mathcal O(n r)$ without side information. Similarly, the number of variables scale as $d$ rather than as $n$, enabling more efficient computation and less memory.
Finally, features also allow completion of rows and columns of $X^*$ that do not have even a single observed entry.
Unlike standard matrix completion which requires at least $r$ observed entries in each row and column of $X^*$, in IMC the feature vector is sufficient to inductively predict the full corresponding row/column; hence the name 'Inductive Matrix Completion'.

\renewcommand{\arraystretch}{1.5}
\begin{table}[t]
\caption{Recovery guarantees for the algorithms: \texttt{Maxide} \cite{xu2013speedup}, \texttt{AltMin} \cite{jain2013provable}, \texttt{MPPF} \cite{zhang2018fast} and \GNIMC (this work), for an $n\times n$ matrix $X^*$ of rank $r$ and condition number $\kappa$, and $d\times d$ side information matrices of incoherence $\mu$, given a fixed target accuracy. Here $f(\kappa, \mu)$ is some function of $\kappa$ and $\mu$. For a more detailed comparison, see \cref{sec:theory_comparison}.}
\centering
 \begin{tabular}{c  c  c  c  c}
 \hline
 Algorithm & \begin{tabular}{@{}c@{}}Sample complexity \\ $|\Omega| \gtrsim ...$\end{tabular} &   \begin{tabular}{@{}c@{}}Requires \\ incoherent $X^*$?\end{tabular} & \begin{tabular}{@{}c@{}}Error \\ decay rate\end{tabular} & \begin{tabular}{@{}c@{}}Time complexity \\ $\sim \mathcal O(...)$ \end{tabular} \\
 \hline
 \texttt{Maxide} & $\mu^2 d r [1+\log(d/r)] \log n$ & yes & unspecified & unspecified \\
 \hline
 \texttt{AltMin} & $\kappa^2 \mu^4 d^2 r^3 \log n$ & no & unspecified & unspecified \\
 \hline
 \texttt{MPPF} & $(\kappa r + d) \kappa^2 \mu^2 r^2 \log d \log n$ & yes & linear & $f(\kappa, \mu)\cdot n^{3/2} d^2 r^3 \log d \log n$ \\
 \hline
 \texttt{GNIMC} (ours) & $\mu^2 d^2 \log n$ & no & quadratic & $ \mu^2 d^3 r \log n$ \\
 \hline
 \end{tabular}
 \label{table:theory_summary}
\end{table}

Several IMC methods were devised in the past years. Perhaps the most popular ones are nuclear norm minimization \cite{xu2013speedup,lu2018prediction} and alternating minimization \cite{jain2013provable,natarajan2014inductive,zhong2015efficient,chen2018predicting}. Another recent method is multi-phase Procrustes flow \cite{zhang2018fast}. While nuclear norm minimization enjoys strong recovery guarantees, it is computationally slow. Other methods are faster, but the number of observed entries for their recovery guarantees to hold depends on the condition number of $X^*$.

In this work, we make three contributions to the IMC problem.
First, by deriving an RIP (Restricted Isometry Property) guarantee for IMC, we prove that under certain conditions the optimization landscape of IMC is benign (\cref{thm:IMC_landscape}).
Compared to a similar result derived in \cite{ghassemi2018global}, our guarantee requires significantly milder conditions, and in addition, addresses the vanilla IMC problem rather than a suitably regularized one.

Second, we propose a simple scheme to estimate the rank of $X^*$ from its observed entries and the side information matrices $A,B$. We also provide a theoretical guarantee for the accuracy of the estimated rank (\cref{thm:rankEstimate}), which holds for either exactly or approximately low rank $X^*$ and with noisy measurements.

Third, we propose a simple Gauss-Newton based method to solve the IMC problem, that is both fast and enjoys strong recovery guarantees. Our algorithm, named \GNIMC (Gauss-Newton IMC), is an adaptation of the \GNMR algorithm \cite{zilber2022gnmr} to IMC.
At each iteration, \GNIMC solves a least squares problem; yet, its per-iteration complexity is of \textit{the same order as gradient descent}. As a result, empirically, our tuning-free \GNIMC implementation is 2 to 17 times faster than competing algorithms in various settings, including ill-conditioned matrices and very few observations, close to the information limit.

On the theoretical front, we prove that given a standard incoherence assumption on $A,B$ and sufficiently many observed entries sampled uniformly at random, \GNIMC recovers $X^*$ at a \textit{quadratic} convergence rate (\cref{thm:GNIMC_guarantee}). As far as we know, this is the only available quadratic convergence rate guarantee for any IMC algorithm.
In addition, we prove that \GNIMC is stable against small arbitrary additive error (\cref{thm:GNIMC_guarantee_noisy}), which may originate from (i) inaccurate measurements of $X^*$, (ii) inaccurate side information, and/or (iii) $X^*$ being only approximately low rank.

Remarkably, our guarantees do not require $X^*$ to be incoherent, and the required number of observations depends only on properties of $A,B$ and not on those of $X^*$.
Other guarantees have similar dependence on $A,B$, but in addition either depend on the condition number of $X^*$ and/or require incoherence of $X^*$, see \cref{table:theory_summary}.
Relaxing the incoherence assumption on $X^*$ is important, since $X^*$ is only partially observed and such an assumption cannot be verified. 
In contrast, the matrices $A,B$ are known and their incoherence can be verified.


{\bf Notation.}
The $i$-th largest singular value of a matrix $X$ is denoted by $\sigma_i = \sigma_i(X)$.
The condition number of a rank-$r$ matrix is denoted by $\kappa = \sigma_1/\sigma_r$.
The $i$-th standard basis vector is denoted by $e_i$, and the Euclidean norm of a vector $x$ by $\|x\|$.
The spectral norm of a matrix $X$ is denoted by $\|X\|_2$, its Frobenius norm by $\|X\|_F$,
its largest row norm by $\|X\|_{2,\infty} \equiv \max_i \|X^\top e_i\|$,
its largest entry magnitude by $\|X\|_\infty \equiv \max_{i,j} |X_{ij}|$,
and the set of its column vectors by $\text{col}(X)$.
A matrix $X$ is an isometry if $X^\top X = I$, where $I$ is the identity matrix.
Denote by $\mathcal P_{AB}: \mathbb R^{n_1\times n_2} \to \mathbb R^{n_1\times n_2}$ the projection operator into the row and column spaces of $A,B$, respectively, such that $\mathcal P_{AB}(X) = AA^\top X BB^\top$ if $A,B$ are isometries.
Denote by $\mathcal P_\Omega: \mathbb R^{n_1\times n_2} \to \mathbb R^{n_1\times n_2}$ the sampling operator that projects a matrix in $\mathbb R^{n_1\times n_2}$ onto an observation set $\Omega \subseteq [n_1]\times [n_2]$, such that $[\mathcal P_\Omega(X)]_{ij} = X_{ij}$ if $(i,j)\in \Omega$ and $0$ otherwise. Denote by $\text{Vec}_\Omega(X) \in \mathbb R^{|\Omega|}$ the vector with the entries $X_{ij}$ for all $(i,j)\in \Omega$. Finally, denote by $p = |\Omega|/(n_1n_2)$ the sampling rate of $\Omega$.

\section{Problem Formulation} \label{sec:problem}
Let $X^* \in \mathbb R^{n_1\times n_2}$ be a matrix of rank $r$.
For now we assume $r$ is known; in \cref{sec:rankEstimate} we present a scheme to estimate $r$, and prove its accuracy.
Assume $\Omega \subseteq [n_1]\times [n_2]$ is uniformly sampled and known, and let $Y = \mathcal P_\Omega(X^* + \mathcal E)$ be the observed matrix where $\mathcal E$ is additive error.
In the standard matrix completion problem, the goal is to solve
\begin{align}\label{eq:MC}
\tag{MC}
\min_X \|\mathcal P_\Omega(X) - Y\|_F^2 \quad \text{s.t. } \text{rank}(X) \leq r .
\end{align}

In IMC, in addition to the observations $Y$ we are given two side information matrices $A \in \mathbb R^{n_1\times d_1}$ and $B \in \mathbb R^{n_2\times d_2}$ with $r \leq d_i \leq n_i$ for $i=1,2$, such that
\begin{align}\label{eq:sideInformation}
\text{col}(X^*) \subseteq \text{span col}(A),\;\; \text{col}(X^{*\top}) \subseteq \text{span col}(B) .
\end{align}
Note that w.l.o.g., we may assume that $A$ and $B$ are isometries, $A^\top A = I_{d_1}$ and $B^\top B = I_{d_2}$, as property \eqref{eq:sideInformation} is invariant to orthonormalization of the columns of $A$ and $B$.
Standard matrix completion corresponds to $d_i = n_i$ with the trivial side information $A = I_{n_1}$, $B = I_{n_2}$.
A common assumption in IMC is $d_i \ll n_i$, so that the side information is valuable.
Note that beyond allowing for (potentially adversarial) inaccurate measurements, $\mathcal E$ may also capture violations of the low rank and the side information assumption \eqref{eq:sideInformation}, as we can view $X^* + \mathcal E$ as the true underlying matrix whose only first component, $X^*$, has exact low rank and satisfies \eqref{eq:sideInformation}.

Assumption \eqref{eq:sideInformation} implies that $X^* = AM^*B^\top$ for some rank-$r$ matrix $M^*\in \mathbb R^{d_1\times d_2}$. The IMC problem thus reads
\begin{align}
\tag{IMC} \label{eq:IMC}
&\min_M \|\mathcal P_\Omega(AMB^\top) - Y\|_F^2 \quad \text{s.t. } \text{rank}(M) \leq r .
\end{align}
Some works on IMC \cite{xu2013speedup,zhang2018fast} assume that both $X^*$ and $A,B$ are incoherent, namely have small incoherence, defined as follows \cite{candes2009exact,keshavan2010matrix}.
\begin{definition}[$\mu$-incoherence]\label{def:incoherence}
A matrix $X \in \mathbb R^{n_1\times n_2}$ of rank $r$ is $\mu$-incoherent if its Singular Value Decomposition (SVD), $U \Sigma V^\top$ with $U \in \mathbb R^{n_1\times r}$ and $V \in \mathbb R^{n_2\times r}$, satisfies
\begin{align*}
\|U\|_{2,\infty} \leq \sqrt{\mu r/n_1} \,\,\mbox{ and }\, 
\|V\|_{2,\infty} \leq \sqrt{\mu r/n_2}.
\end{align*}
\end{definition}
However, for IMC to be well-posed, $X^*$ does not have to be incoherent, and it suffices for $A,B$ to be incoherent \cite{jain2013provable}.
In case $A$ and $B$ are isometries, their incoherence assumption corresponds to bounded row norms, $\|A\|_{2,\infty} \leq \sqrt{\mu d_1/n_1}$ and $\|B\|_{2,\infty} \leq \sqrt{\mu d_2/n_2}$. 

\section{No Bad Local Minima Guarantee} \label{sec:landscape}
In this section we present a novel characterization of the optimization landscape of IMC.
Following the factorization approach to matrix recovery problems, we first incorporate the rank constraint into the objective by writing the unknown matrix as $M = UV^\top$ where $U \in \mathbb R^{d_1\times r}$ and $V \in \mathbb R^{d_2\times r}$.
Then, problem \eqref{eq:IMC} is 
\begin{align}\label{eq:IMC_factorized}
\min_{U,V} \|\mathcal P_\Omega(AUV^\top B^\top) -Y\|_F^2 .
\end{align}
Clearly, any pair of matrices $(U, V)$ whose product is $UV^\top = M^*$ is a global minimizer of \eqref{eq:IMC_factorized} with an objective value of zero.
However, as \eqref{eq:IMC_factorized} is non-convex, some of its first-order critical points, namely points at which the gradient vanishes, may be bad local minima.
The next result, proven in \cref{sec:proof_RIP_consequences}, states that if sufficiently many entries are observed, all critical points are either global minima or strict saddle points. At a strict saddle point the Hessian has at least one strictly negative eigenvalue, so that gradient descent will not reach it.
Hence, under the conditions of \cref{thm:IMC_landscape}, gradient descent will recover $M^*$ from a random initialization.

\begin{theorem}\label{thm:IMC_landscape}
Let $X^* \in \mathbb R^{n_1\times n_2}$ be a rank-$r$ matrix which satisfies \eqref{eq:sideInformation} with $\mu$-incoherent matrices $A \in \mathbb R^{n_1\times d_1}$ and $B \in \mathbb R^{n_2\times d_2}$.
Assume $\Omega \subseteq [n_1]\times [n_2]$ is uniformly sampled with $|\Omega| \gtrsim \mu^2 d_1 d_2 \log n$.
Then w.p.~at least $1-2n^{-2}$, any critical point $(U,V)$ of problem \eqref{eq:IMC_factorized} is either a global minimum with $UV^\top = M^*$, or a strict saddle point.
\end{theorem}
To the best of our knowledge, \cref{thm:IMC_landscape} is the first guarantee for the geometry of vanilla IMC. A previous result by \cite{ghassemi2018global} only addressed a suitably balance-regularized version of \eqref{eq:IMC_factorized}. In addition, their guarantee requires $\mathcal O(\mu^2 r \max\{d_1,d_2\} \max\{d_1d_2, \log^2 n\})$ observed entries with cubic scaling in $d_1,d_2$,\footnote{Note the notation in \cite{ghassemi2018global} is slightly different than ours; see \cref{sec:comparisonToGhassemi2018} for more details.} which is significantly larger than the quadratic scaling in our \cref{thm:IMC_landscape}.

\Cref{thm:IMC_landscape} guarantees exact recovery for a family of algorithms beyond vanilla gradient descent. However, as illustrated in \cref{sec:experiments}, solving the IMC problem can be done much faster than by gradient descent or variants thereof, e.g.~by our proposed \GNIMC method described in \cref{sec:GNIMC}.

\subsection{IMC as a special case of matrix sensing} \label{sec:theory_connection}
Similar to \cite{ghassemi2018global}, our proof of \cref{thm:IMC_landscape} is based on an RIP result we derive for IMC. The RIP result forms a connection between IMC and the matrix sensing (MS) problem, as follows.
Recall that in IMC, the goal is to recover $M^* \in \mathbb R^{d_1\times d_2}$ from the observations $Y = \mathcal P_\Omega(AM^*B^\top + \mathcal E)$.
In MS, we observe a set of linear measurements $b \equiv \mathcal A(M^*) + \xi$ where $\mathcal A: \mathbb R^{d_1\times d_2}\to \mathbb R^m$ is a sensing operator and $\xi\in \mathbb R^m$ is additive error. Assuming a known or estimated rank $r$ of $M^*$, the goal is to solve
\begin{align}
\tag{MS}\label{eq:MS}
\min_M \|\mathcal A(M) - b\|^2 \quad &\text{s.t. } \text{rank}(M) \leq r .
\end{align}
Problem \eqref{eq:IMC} is in the form of \eqref{eq:MS} with the operator
\begin{align}\label{eq:sensingOperator_IMC}
\mathcal A(M) = \text{Vec}_\Omega (AMB^\top) / \sqrt p
\end{align}
and the error vector $\xi = \text{Vec}_\Omega(\mathcal E)/\sqrt p$.
However, unlike IMC, in MS the operator $\mathcal A$ is assumed to satisfy a suitable RIP (Restricted Isometry Property), defined as follows \cite{candes2008restricted,recht2010guaranteed}.
\begin{definition}\label{def:RIP}
A linear map $\mathcal A: \mathbb R^{d_1\times d_2}\to \mathbb R^m$ satisfies a $k$-RIP with a constant $\delta \in [0,1)$, if for all matrices $M \in \mathbb R^{d_1\times d_2}$ of rank at most $k$,
\begin{align}\label{eq:RIP}
(1-\delta) \|M\|_F^2 \leq \|\mathcal A(M)\|^2 \leq (1+\delta) \|M\|_F^2 .
\end{align}
\end{definition}
The following theorem, proven in \cref{sec:proof_IMC_RIP}, states that if $A,B$ are incoherent and $|\Omega|$ is sufficiently large, w.h.p.~the IMC sensing operator \eqref{eq:sensingOperator_IMC} satisfies the RIP. This observation creates a bridge between IMC and MS:
for a given MS method, its RIP-based theoretical guarantees can be directly transferred to IMC.
\begin{theorem}\label{thm:IMC_RIP}
Let $A \in \mathbb R^{n_1\times d_1}$, $B \in \mathbb R^{n_2\times d_2}$ be two isometry matrices such that $\|A\|_{2,\infty} \leq \sqrt{\mu d_1/n_1}$ and $\|B\|_{2,\infty} \leq \sqrt{\mu d_2/n_2}$.
Let $\delta \in [0,1)$, and assume $\Omega \subseteq [n_1]\times [n_2]$ is uniformly sampled with $|\Omega| \equiv m \geq (8/\delta^2) \mu^2 d_1 d_2 \log n$.
Then, w.p.~at least $1-2n^{-2}$, the sensing operator $\mathcal A$ defined in \eqref{eq:sensingOperator_IMC} satisfies an RIP \eqref{eq:RIP} with $k = \min\{d_1,d_2\}$ and with constant $\delta$.
\end{theorem}
A similar result was derived in \cite{ghassemi2018global}. \Cref{thm:IMC_RIP} improves upon it both in terms of the required conditions and in terms of the RIP guarantee. First, as in their landscape guarantee, \cite{ghassemi2018global} require  cubic scaling with $d_1,d_2$ rather than quadratic as in our result. Moreover, their sample complexity includes an additional factor of $r\log(1/\delta)$ (see \cref{sec:comparisonToGhassemi2018}). Second, they proved only a $\min\{2r,d_1,d_2\}$-RIP, whereas \cref{thm:IMC_RIP} guarantees that $\mathcal A$ satisfies the RIP with the \textit{maximal} possible rank $\min\{d_1,d_2\}$.
In particular, this allows us to employ a recent result due to \cite{li2020global} to prove \cref{thm:IMC_landscape} for vanilla IMC.
The technical reason behind our sharper results is that instead of applying the Bernstein matrix inequality to a fixed matrix and then proving a union bound for all matrices, we apply it to a cleverly designed operator, which directly guarantees the result for all matrices (see \cref{lem:AB_RIP}).

\section{Rank Estimation Scheme} \label{sec:rankEstimate}
The factorization approach \eqref{eq:IMC_factorized} requires knowing $r$ in advance, although in practice it is often unknown. In this section we propose a simple scheme to estimate the underlying rank, and provide a theoretical guarantee for it. Importantly, our scheme does not assume $X^*$ is exactly low rank, but rather the existence of a sufficiently large spectral gap between its $r$-th and $(r+1)$-th singular values.

Let $\hat X = \mathcal P_{AB}(Y)/p = AA^\top Y BB^\top /p$ where $Y$ is the observed matrix and $p \equiv |\Omega|/(n_1n_2)$, and denote its singular values by $\hat \sigma_i$. Our estimator for the rank of $X^*$ is
\begin{align}\label{eq:rankEstimate_scheme}
\hat{r} &= \argmax_i \, g_i(\hat X), \quad
g_i(\hat X) = \frac{\hat \sigma_i}{\hat \sigma_{i+1} + D\cdot \hat \sigma_1 \sqrt{i}},
\end{align}
for some constant $D < 1$. In our simulations we set $D = (\sqrt{d_1d_2}/|\Omega|)^{1/2}$. The function $g_i$ measures the $i$-th spectral gap, with the second term in the denominator added for robustness of the estimate. For $D=0$, $g_i$ is simply the ratio between two consecutive singular values.
A similar estimator was proposed in \cite{keshavan2009low} for standard matrix completion, though they did not provide guarantees for it. The difference in our estimator is the incorporation of the side information matrices $A,B$. In addition, we present the following theoretical guarantee for our estimator, proven in \cref{sec:proof_rankEstimate}. Note that using the side information matrices $A,B$ allows us to reduce the sample complexity from $\mathcal O(n)$, as necessary in standard matrix completion, to only $\mathcal O(\log(n))$.

\begin{theorem}\label{thm:rankEstimate}
There exists a sufficiently small constant $c$ such that the following holds w.p.~at least $1-2n^{-2}$.
Let $X^* \in \mathbb R^{n_1\times n_2}$ be a matrix which satisfies \eqref{eq:sideInformation} with $\mu$-incoherent $A,B$. Assume $X^*$ is approximately rank $r$, in the sense that for all $i\neq r$, $g_r(X^*) > \min\{(11/10) g_i(X^*), 1/10\}$.
Denote $\delta = \min_i \{\sigma_{i+1}(X^*) + D\sigma_1(X^*) \sqrt{i}\}$, and assume $\Omega \subseteq [n_1]\times [n_2]$ is uniformly sampled with $|\Omega| \geq 8\mu^2 d_1 d_2 \log(n) \|X\|_F^2 / (c\delta)^2$. Further assume bounded error $\epsilon \equiv \|\mathcal P_{AB}\mathcal P_\Omega(\mathcal E)\|_F/p \leq c\delta$.
Then $\hat{r} = r$.
\end{theorem}

To the best of our knowledge, \cref{thm:rankEstimate} is the first guarantee in the literature for rank estimation in IMC.
We remark that with a suitably modified $\delta$, our guarantee holds for other choices of $g_i$ as well (including $g_i = \sigma_i/\sigma_{i+1}$, corresponding to $D=0$). An empirical demonstration of our scheme appears in \cref{sec:experiments_rankEstimate}.

\section{GNIMC Algorithm} \label{sec:GNIMC}
In this section, we describe an adaptation of the \GNMR algorithm \cite{zilber2022gnmr} to IMC, and present recovery guarantees for it.
Consider the factorized objective \eqref{eq:IMC_factorized}.
Given an estimate $(U, V)$, the goal is to find an update $(\Delta U, \Delta V)$ such that $(U', V') = (U + \Delta U, V + \Delta V)$ minimizes \eqref{eq:IMC_factorized}.
In terms of $(\Delta U, \Delta V)$, problem \eqref{eq:IMC_factorized} reads
\begin{align*}
\min_{\Delta U, \Delta V} \|&\mathcal P_\Omega(AUV^\top B^\top + AU \Delta V^\top B^\top + A\Delta U V^\top B^\top + A\Delta U \Delta V^\top B^\top) - Y \|_F^2 ,
\end{align*}
which is nonconvex due to the mixed term $\Delta U \Delta V^\top$.
The Gauss-Newton approach is to neglect this term. This yields the key iterative step of \GNIMC, which is solving the following sub-problem:
\begin{align}\label{eq:GNIMC_LSQR}
\min_{\Delta U, \Delta V} \|& \mathcal P_\Omega(AUV^\top B^\top + AU \Delta V^\top B^\top + A\Delta U V^\top B^\top) - Y \|_F^2.
\end{align}
Problem \eqref{eq:GNIMC_LSQR} is a linear least squares problem.
Note, however, that it has an infinite number of solutions: for example, if $(\Delta U, \Delta V)$ is a solution, so is $(\Delta U + U R, \Delta V - VR^\top)$ for any $R \in \mathbb R^{r\times r}$.
We choose the solution with minimal norm $\|\Delta U\|_F^2 + \|\Delta V\|_F^2$, see \cref{alg:GNIMC}. In practice, this solution can be computed using the standard LSQR algorithm \cite{paige1982lsqr}.

\begin{algorithm}[t]
\caption{\GNIMC} \label{alg:GNIMC}
\SetKwInOut{Return}{return}
\SetKwInOut{Input}{input}
\SetKwInOut{Output}{output}
\Input{
sampling operator $\mathcal P_\Omega$, observed matrix $Y$, side information matrices $(A,B)$, maximal number of iterations $T$, initialization $(U_0, V_0)$
}
\Output{rank-$r$ (approximate) solution to $\mathcal P_\Omega(\hat X) = Y$}
\For{$t=0,\ldots,T-1$}{
set $\begin{psmallmatrix} U_{t+1} \\ V_{t+1} \end{psmallmatrix} = \begin{psmallmatrix} U_{t} \\ V_{t} \end{psmallmatrix} + \begin{psmallmatrix} \Delta U_{t+1} \\ \Delta V_{t+1} \end{psmallmatrix}$, where $\begin{psmallmatrix} \Delta U_{t+1} \\ \Delta V_{t+1} \end{psmallmatrix}$ is the minimal norm solution of
	$ \argmin_{\Delta U, \Delta V} \| \mathcal P_\Omega[A(U_tV_t^\top + U_t \Delta V^\top + \Delta U V_t^\top)B^\top] - Y \|_F^2 $
}
\Return{$\hat X = A U_T V_T^\top B^\top$}
\end{algorithm}

In general, the computational complexity of solving problem \eqref{eq:GNIMC_LSQR} scales with the condition number $\kappa$ of $X^*$. To decouple the runtime of \GNIMC from $\kappa$, we use the QR decompositions of $U_t$ and $V_t$ as was similarly done for alternating minimization by \cite{jain2013low}. In \cref{sec:time_complexity} we describe the full procedure, and prove it is analytically equivalent to \eqref{eq:GNIMC_LSQR}. Remarkably, despite the fact that \GNIMC performs a non-local update at each iteration, its resulting per-iteration complexity is as low as a single gradient descent step.

\GNIMC requires an initial guess $(U_0,V_0)$. A suitable initialization procedure for our theoretical guarantees is discussed in \cref{proposition:initialization}. In practice, \GNIMC works well also from a random initialization.

The proposed \GNIMC algorithm is extremely simple, as it merely solves a least squares problem in each iteration.
In contrast to several previous methods, it requires no parameter estimation such as the minimal and maximal singular values of $X^*$, or tuning of hyperparameters such as regularization coefficients. Altogether, this makes \GNIMC easy to implement and use. Furthermore, \GNIMC enjoys strong recovery guarantees and fast runtimes, as described below.

\subsection{Recovery guarantees for GNIMC} \label{sec:GNIMC_guarantee}
We first analyze the noiseless case, $\mathcal E = 0$.
The following theorem, proven in \cref{sec:proof_RIP_consequences}, states that starting from a sufficiently accurate initialization with small imbalance $\|U^\top U-V^\top V\|_F$, \GNIMC exactly recovers the matrix at a quadratic rate.
In fact, the balance condition can be eliminated by adding a single SVD step as discussed below.

\begin{theorem}\label{thm:GNIMC_guarantee}
There exists a constant $c > 1$ such that the following holds w.p.~at least $1-2n^{-2}$.
Let $X^* \in \mathbb R^{n_1\times n_2}$ be a rank-$r$ matrix which satisfies \eqref{eq:sideInformation} with $\mu$-incoherent side matrices $A \in \mathbb R^{n_1\times d_1}$ and $B \in \mathbb R^{n_2\times d_2}$.
Denote $\gamma = c/(2\sigma_r^*)$ where $\sigma_r^* = \sigma_r(X^*)$.
Assume $\Omega \subseteq [n_1]\times [n_2]$ is uniformly sampled with
\begin{align}\label{eq:GNIMC_guarantee_sampleComplexity}
|\Omega| \geq 32 \mu^2 d_1 d_2 \log n.
\end{align}
Then, for any initial iterate $(U_0, V_0)$ that satisfies
\begin{subequations}
\begin{align}
\|AU_0V_0^\top B^\top - X^*\|_F &\leq \frac{\sigma_r^*}{c}, \label{eq:initialization_accuracy} \\
\|U_0^\top U_0 - V_0^\top V_0\|_F &\leq \frac{\sigma_r^*}{2c}, \label{eq:initialization_balance}
\end{align}
\end{subequations}
the estimates $X_t = AU_tV_t^\top B^\top$ of \cref{alg:GNIMC} satisfy
\begin{align}\label{eq:GNIMC_guarantee}
\|X_{t+1} - X^*\|_F \leq \gamma\cdot \|X_t - X^*\|_F^2, \hspace{0.12in} \forall t=0, 1, ... .
\end{align}
\end{theorem}
Note that by assumption \eqref{eq:initialization_accuracy}, $\gamma\cdot \|X_0 - X^*\|_F \leq 1/2$. Hence, \eqref{eq:GNIMC_guarantee} implies that \GNIMC achieves exact recovery, since $X_t \to X^*$ as $t\to\infty$.
The computational complexity of \GNIMC is provided in the following proposition, proven in \cref{sec:time_complexity}.
\begin{proposition}\label{proposition:time_complexity}
Under the conditions of \cref{thm:GNIMC_guarantee}, the time complexity of \GNIMC (\cref{alg:GNIMC}) until recovery with a fixed accuracy (w.h.p.) is $\mathcal O(\mu^2 (d_1+d_2) d_1 d_2 r \log n)$.
\end{proposition}
To meet the initialization conditions of \cref{thm:GNIMC_guarantee}, we need to find a rank-$r$ matrix $M$ which satisfies $\|AMB^\top - X^*\| \leq \sigma_r^*/c$. By taking its SVD $M = U\Sigma V^\top$, we obtain that $(U \Sigma^\frac{1}{2}, V \Sigma^\frac{1}{2})$ satisfies conditions (\ref{eq:initialization_accuracy}-\ref{eq:initialization_balance}).
Such a matrix $M$ can be computed in polynomial time
using the initialization procedure suggested in \cite{tu2016low} for matrix sensing. Starting from $M_0 = 0$, it iteratively performs a gradient descent step and projects the result into the rank-$r$ manifold. Its adaptation to IMC reads
\begin{align}
M_{\tau+1}
&= \mathcal P_r \left[ M_\tau - A^\top (\mathcal P_\Omega(AM_\tau B^\top)/p - Y) B \right] \label{eq:initialization_proceudre}
\end{align}
where $\mathcal P_r(M)$ is the rank-$r$ truncated SVD of $M$.
The following proposition, proven in \cref{sec:proof_initialization}, states that $\mathcal O \left(\log (r \kappa)\right)$ iterations suffice to meet the initialization conditions of \cref{thm:GNIMC_guarantee} under a slightly larger sample size requirement.

\begin{proposition}[Initialization guarantee] \label{proposition:initialization}
Let $X^*, A, B$ be as in \cref{thm:GNIMC_guarantee}. Assume $\Omega$ is uniformly sampled with
$|\Omega| \geq 50\mu^2 d_1d_2\log n$.
Let $M_\tau$ be the result after $\tau \geq 5\log(c\sqrt r\kappa)$ iterations of \eqref{eq:initialization_proceudre}, and denote its SVD by $U \Sigma V$. Then w.p.~$1-2n^{-2}$, $\begin{psmallmatrix} U_0 \\ V_0 \end{psmallmatrix} = \begin{psmallmatrix} U \Sigma^\frac{1}{2} \\ V \Sigma^\frac{1}{2} \end{psmallmatrix}$ satisfies the initialization conditions \eqref{eq:initialization_accuracy}-\eqref{eq:initialization_balance} of \cref{thm:GNIMC_guarantee}.
\end{proposition}


We conclude this subsection with a guarantee for \GNIMC in the noisy setting. Suppose we observe $Y = \mathcal P_\Omega(X^* + \mathcal E)$ where $\mathcal E$ is arbitrary additive error.
To cope with the error, we slightly modify \cref{alg:GNIMC}, and add the following balancing step at the start of each iteration: calculate the SVD $\bar U \Sigma \bar V^\top$ of the current estimate $U_t V_t^\top$, and update
\begin{align}\label{eq:balancing_step}
U_t \leftarrow \bar U \Sigma^\frac{1}{2}, \quad
V_t \leftarrow \bar V \Sigma^\frac{1}{2},
\end{align}
so that $(U_t, V_t)$ are perfectly balanced with $U_t^\top U_t = V_t^\top V_t$.
The following result holds for the modified algorithm.
\begin{theorem}\label{thm:GNIMC_guarantee_noisy}
Let $X^*, A,B, \Omega$ be defined as in \cref{thm:GNIMC_guarantee}, and suppose the error is bounded as
\begin{align}
\epsilon \equiv \tfrac{1}{\sqrt p} \|\mathcal P_\Omega(\mathcal E)\|_F \leq \frac{\sigma_r^*}{9c} .
\end{align}
Then for any initial iterate $(U_0, V_0)$ that satisfies \eqref{eq:initialization_accuracy}, the estimates $X_t = AU_t V_t^\top B^\top$ of \cref{alg:GNIMC} with the balancing step \eqref{eq:balancing_step} satisfy
\begin{align}
\|X_t - X^*\|_F \leq \frac{\sigma_r^*}{4^{2^t-1} c} + 6\epsilon \stackrel{t\to\infty}{\longrightarrow} 6\epsilon.
\end{align}
\end{theorem}
In the absence of errors, $\epsilon = 0$, this result reduces to the exact recovery guarantee with quadratic rate of \cref{thm:GNIMC_guarantee}.

\subsection{Comparison to prior art}\label{sec:theory_comparison}

Here we describe recovery guarantees for three other algorithms. We compare them only to \cref{thm:GNIMC_guarantee}, as none of these works derived a stability to error result analogous to our \cref{thm:GNIMC_guarantee_noisy}.
A summary appears in \cref{table:theory_summary}.
In the following, let $n = \max\{n_1,n_2\}$ and $d = \max\{d_1,d_2\}$.
For works which require incoherence condition on several matrices, we use for simplicity the same incoherence coefficient $\mu$. All guarantees are w.p.~at least $1-\mathcal O(1/n)$.

\textbf{Nuclear norm minimization (Maxide)} \cite{xu2013speedup}. If (i)  both $X^*$ and $A,B$ are $\mu$-incoherent, (ii) $\|LR^\top\|_\infty \leq \mu r/(n_1n_2)$ where $L\Sigma R$ is the SVD of $X^*$, (iii) $d_1d_2 + r^2 \geq 8[1+\log_2(d/r)](d_1+d_2)r$, and (iv)
\begin{align}\label{eq:Maxide_Omega}
|\Omega| \gtrsim \mu^2 r d [1 + \log (d/r)] \log n,
\end{align}
then \texttt{Maxide} exactly recovers $X^*$.

\textbf{Alternating minimization} \cite{jain2013provable}. If $A,B$ are $\mu$-incoherent and
\begin{align}\label{eq:AltMin_Omega}
|\Omega| \gtrsim \kappa^2 \mu^4 r^3 d_1 d_2 \log n \log (1/\epsilon),
\end{align}
then \texttt{AltMin} recovers $X^*$ up to error $\epsilon$ in spectral norm at a linear rate with a constant contraction factor.

\textbf{Multi-phase Procrustes flow} \cite{zhang2018fast}. If both $X^*$ and $A,B$ are $\mu$-incoherent and
\begin{align}\label{eq:MPPF_Omega}
|\Omega| \gtrsim \max\{\kappa r, d\} \kappa^2 \mu^2 r^2 \log d \log n,
\end{align}
then \texttt{MPPF} recovers $X^*$ at a linear rate with a contraction factor smaller than $1 - \mathcal O(1/(r\kappa))$.\footnote{When the estimation error decreases below $\mathcal O(1/(\mu d))$, the contraction factor is improved to $1 - \mathcal O(1/\kappa)$.}
This guarantee implies a required number of iterations which may scale linearly with $\kappa$, as is indeed empirically demonstrated in \cref{fig:convergence_recoveryVsCN}(right).

Notably, in terms of the dimensions $n, d, r$, the sample complexity for \texttt{Maxide} \eqref{eq:Maxide_Omega} is order optimal up to logarithmic factors. However, their guarantee requires few additional assumptions, including incoherent $X^*$. Also, from a practical point of view, \texttt{Maxide} is computationally slow and not easily scalable to large matrices (see \cref{fig:convergence_recoveryVsCN}(left)). In contrast, \GNIMC is computationally much faster and does not require $X^*$ to be incoherent, a relaxation which can be important in practice as discussed in the introduction.
Furthermore, our sample complexity requirement \eqref{eq:GNIMC_guarantee_sampleComplexity} is the only one independent of the condition number without requiring incoherent $X^*$.
Compared to the other factorization-based methods, our sample complexity is strictly better than that of \texttt{AltMin}, and better than \texttt{MPPF} if $\min\{d_1,d_2\} \lesssim \kappa^2 r^2 \log d$. Since $\min\{d_1,d_2\} \leq r^2$ is a practical setting (see e.g.~\cite[Section~4.4]{natarajan2014inductive} and \cite[Sections~6.1-6.2]{zhang2018fast}), our complexity is often smaller than that of \texttt{MPPF} even for well-conditioned matrices.
In fact, if $\min\{d_1,d_2\} \leq 54r$, then our guarantee is the sharpest, as condition (iii) of \texttt{Maxide} is violated.
In addition, to the best of our knowledge, \GNIMC is the only method with a quadratic convergence rate guarantee.
Finally, its contraction factor is constant, and in particular independent of the rank $r$ and the condition number $\kappa$.

We conclude this subsection with a computational complexity comparison.
Among the above works, only the computational complexity of \texttt{MPPF} was analyzed, and it is given by $\mathcal O(f(\kappa, \mu) \cdot n^{3/2} d^2 r^3 \log d \log n)$ where $f(\kappa, \mu)$ is some function of $\kappa$ and $\mu$ which was left unspecified in \cite{zhang2018fast}. The dependence on the large dimension factor $n^{3/2}$ implies that \texttt{MPPF} does not exploit the available side information in terms of computation time.
Our complexity guarantee, \cref{proposition:time_complexity}, is fundamentally better. In particular, it depends on $n$ only logarithmically, and is independent of the condition number $\kappa$. This independence is demonstrated empirically in \cref{fig:convergence_recoveryVsCN}(right).

\section{Simulation results} \label{sec:experiments}

\begin{figure}[t]
\centering
	\subfloat{
		\includegraphics[width=0.48\linewidth]{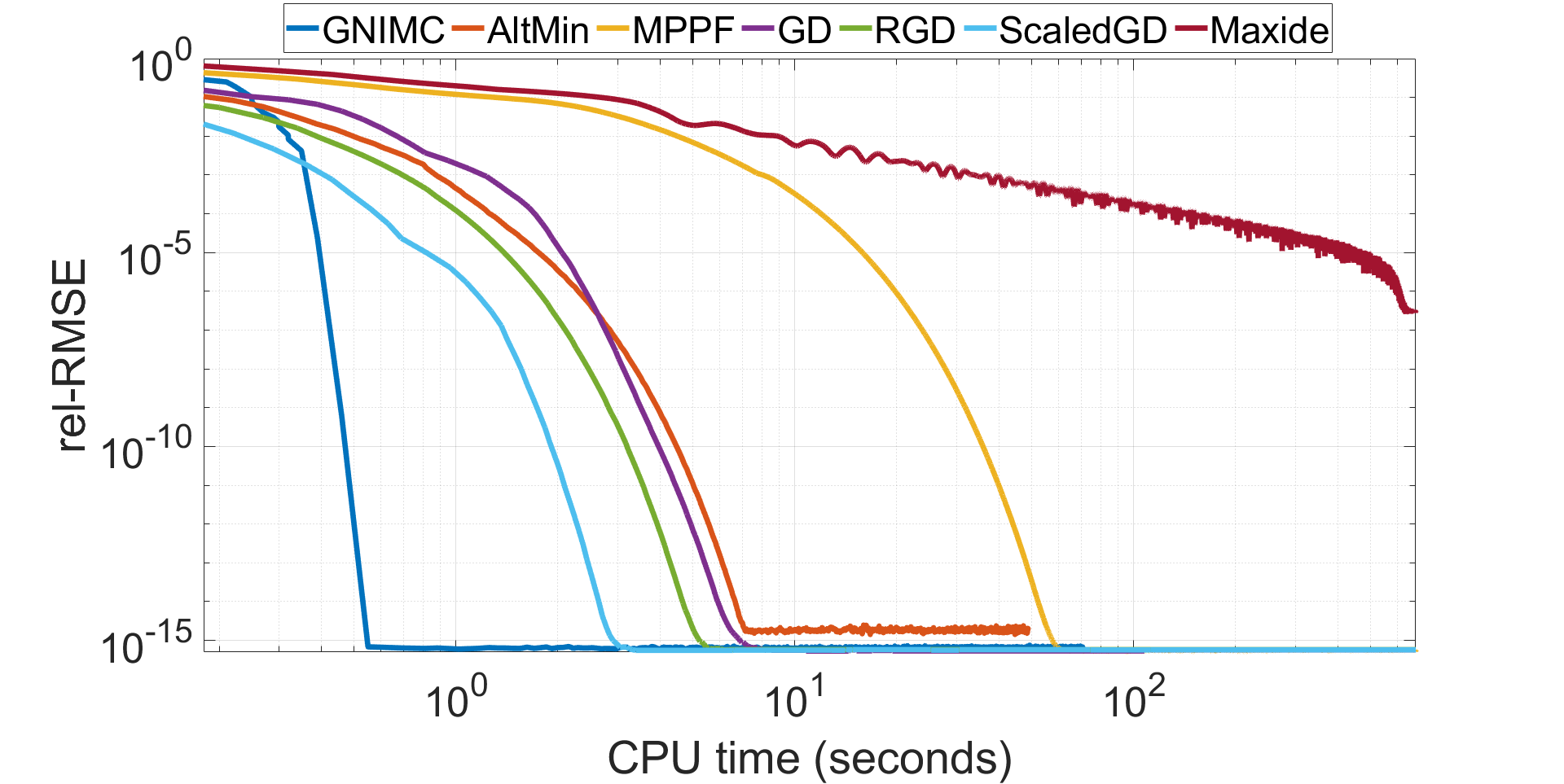}
		\label{fig:errorVsTime_n1000}
	}
	\subfloat{
		 \includegraphics[width=0.48\linewidth]{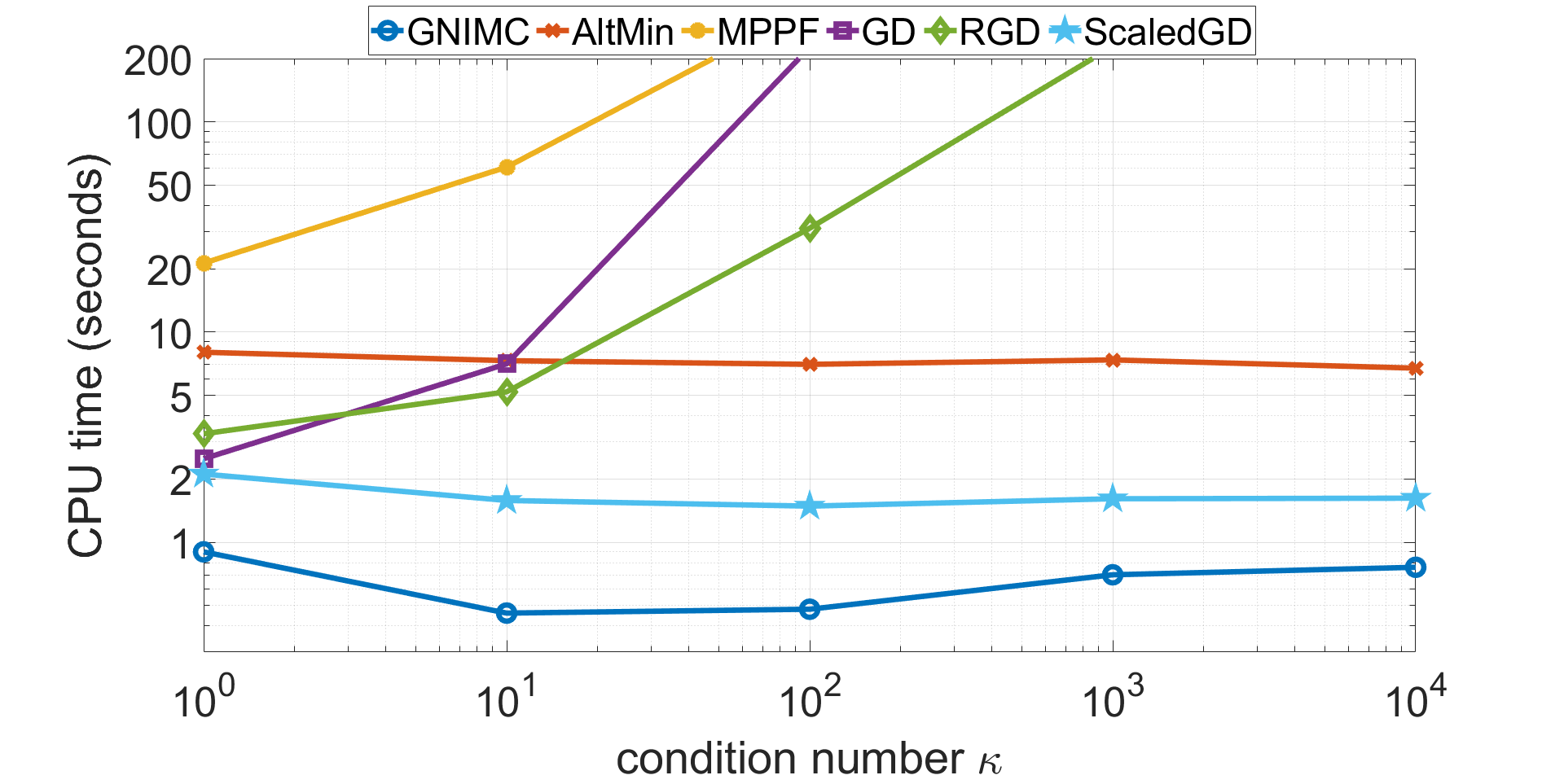}
		\label{fig:timeVsCN_n1000}
 }
 \caption{Left panel: \texttt{rel-RMSE} \eqref{eq:relRMSE} as a function of CPU runtime for several IMC algorithms. Here $X^*$ has a condition number $\kappa=10$. Right panel: runtime till convergence as a function of $\kappa$, where each point corresponds to the median of $50$ independent realizations. In both panels, $X^* \in \mathbb R^{1000\times 1000}$, $A,B\in \mathbb R^{20\times 20}$, $r = 10$ and oversampling ratio $\rho = 1.5$.}
\label{fig:convergence_recoveryVsCN}
\end{figure}

We compare the performance of \GNIMC to the following IMC algorithms, all implemented in MATLAB.\footnote{\label{fn:implementation}Code implementations of \GNIMC, \texttt{AltMin}, \texttt{GD} and \texttt{RGD} are available at \url{github.com/pizilber/GNIMC}.}
\texttt{AltMin} \cite{jain2013provable}: our implementation of alternating minimization including the QR decomposition for reduced runtime;
\texttt{Maxide} \cite{xu2013speedup}: nuclear norm minimization as implemented by the authors;%
\footnote{\url{www.lamda.nju.edu.cn/code_Maxide.ashx}}
\texttt{MPPF} \cite{zhang2018fast}: multi-phase Procrustes flow as implemented by the authors;%
\footnote{\url{github.com/xiaozhanguva/Inductive-MC}}
\texttt{GD}, \texttt{RGD}: our implementations of vanilla gradient descent (\texttt{GD}) and a variant regularized by an imbalance factor $\|U^\top U - V^\top V\|_F$ (\texttt{RGD});
and \texttt{ScaledGD} \cite{tong2021accelerating}: a preconditioned variant of gradient descent.%
\footnote{\url{github.com/Titan-Tong/ScaledGD}. We adapted the algorithm, originally designed for matrix completion, to the IMC problem. In addition, we implemented computations with sparse matrices to enhance its performance.}
Details on initialization, early stopping criteria and a tuning scheme for the hyperparameters of \texttt{Maxide}, \texttt{MPPF}, \texttt{RGD} and \texttt{ScaledGD} appear in \cref{sec:additional_experimental_details}. \GNIMC and \texttt{AltMin} require no tuning. 

In each simulation we construct $U \in \mathbb R^{d_1\times r}$, $V \in \mathbb R^{d_2\times r}$, $A \in \mathbb R^{n_1\times d_1}$, $B \in \mathbb R^{n_2\times d_2}$ with entries i.i.d.~from the standard normal distribution, and orthonormalize their columns. We then set $X^* = AUDV^\top B^\top$ where $D\in \mathbb R^{r\times r}$ is diagonal with entries linearly interpolated between $1$ and $\kappa$.
A similar scheme was used in \cite{zhang2018fast}, with a key difference that we explicitly control the condition number of $X^*$ to study how it affects the performance of the various methods.
Next, we sample $\Omega$ of a given size $|\Omega|$ from the uniform distribution over $[n_1]\times [n_2]$.
Since $A$ and $B$ are known, the $n_1\times n_2$ matrix $X^*$ has only $(d_1+d_2-r)r$ degrees of freedom. Denote the oversampling ratio by $\rho = \frac{|\Omega|}{(d_1+d_2-r)r}$. As $\rho$ is closer to the information limit value of $1$, the more challenging the problem becomes.
Notably, our simulations cover a broad range of settings, including much fewer observed entries and higher condition numbers than previous studies \cite{xu2013speedup,zhang2018fast}.

We measure the quality of an estimate $\hat X$ by its relative RMSE,
\begin{align}\label{eq:relRMSE}
\texttt{rel-RMSE} = \frac{\|X^* - \hat X\|_F}{\|X^*\|_F}.
\end{align}
First, we explore the convergence rate of the various algorithms, by comparing their relative RMSE as a function of runtime, in the setting $n_1 = n_2 = 1000$, $d_1 = d_2 = 20$, $r = \kappa = 10$ and $\rho = 1.5$ (sampling rate $p = 0.045\%$).
Representative results of a single instance of the simulation, illustrating the behavior of the algorithms near convergence, are depicted in \cref{fig:convergence_recoveryVsCN}(left).
As shown in the figure, \GNIMC converges much faster than the competing algorithms due to its quadratic convergence rate.

Next, we examine how the runtime of each algorithm is affected by the number of observations and by the condition number.
The runtime is defined as the CPU time required for the algorithm to (i) converge, namely satisfy one of the stopping criteria (detailed in \cref{sec:additional_experimental_details}), and (ii) achieve $\texttt{rel-RMSE} \leq 10^{-4}$. If the runtime exceeds $20$ minutes without convergence, the run is stopped.

Figures~\ref{fig:convergence_recoveryVsCN}(right) and \ref{fig:timeVsOversampling}(left) show the median recovery time on a log scale as a function of the condition number and of the oversampling ratio, respectively, in the same setting as above. \Cref{fig:timeVsOversampling}(right) corresponds to a larger matrix with $n_1 = 20000$, $n_2 = 1000$, $d_1 = 100$, $d_2 = 50$, $r = 5$ and $\kappa = 10$.
Evidently, under a broad range of conditions, \GNIMC is faster than the competing methods, in some cases by an order of magnitude.
In general, the advantage of \GNIMC with respect to the competing methods is more significant at low oversampling ratios.

Remarkably, the runtime of \GNIMC, \texttt{AltMin} and \texttt{ScaledGD} shows almost no sensitivity to the condition number, as illustrated in \cref{fig:convergence_recoveryVsCN}(right). For \GNIMC, this empirical observation is in agreement with \cref{proposition:time_complexity}, which states that the computational complexity of \GNIMC does not depend on the condition number. In contrast, the runtime of the non-preconditioned gradient descent methods increases approximately linearly with the condition number.

Additional simulation results, including demonstration of the stability of \GNIMC to noise, appear in \cref{sec:additional_experimental_results}.

\begin{figure}[t]
\centering
	\subfloat{
		\includegraphics[width=0.48\linewidth]{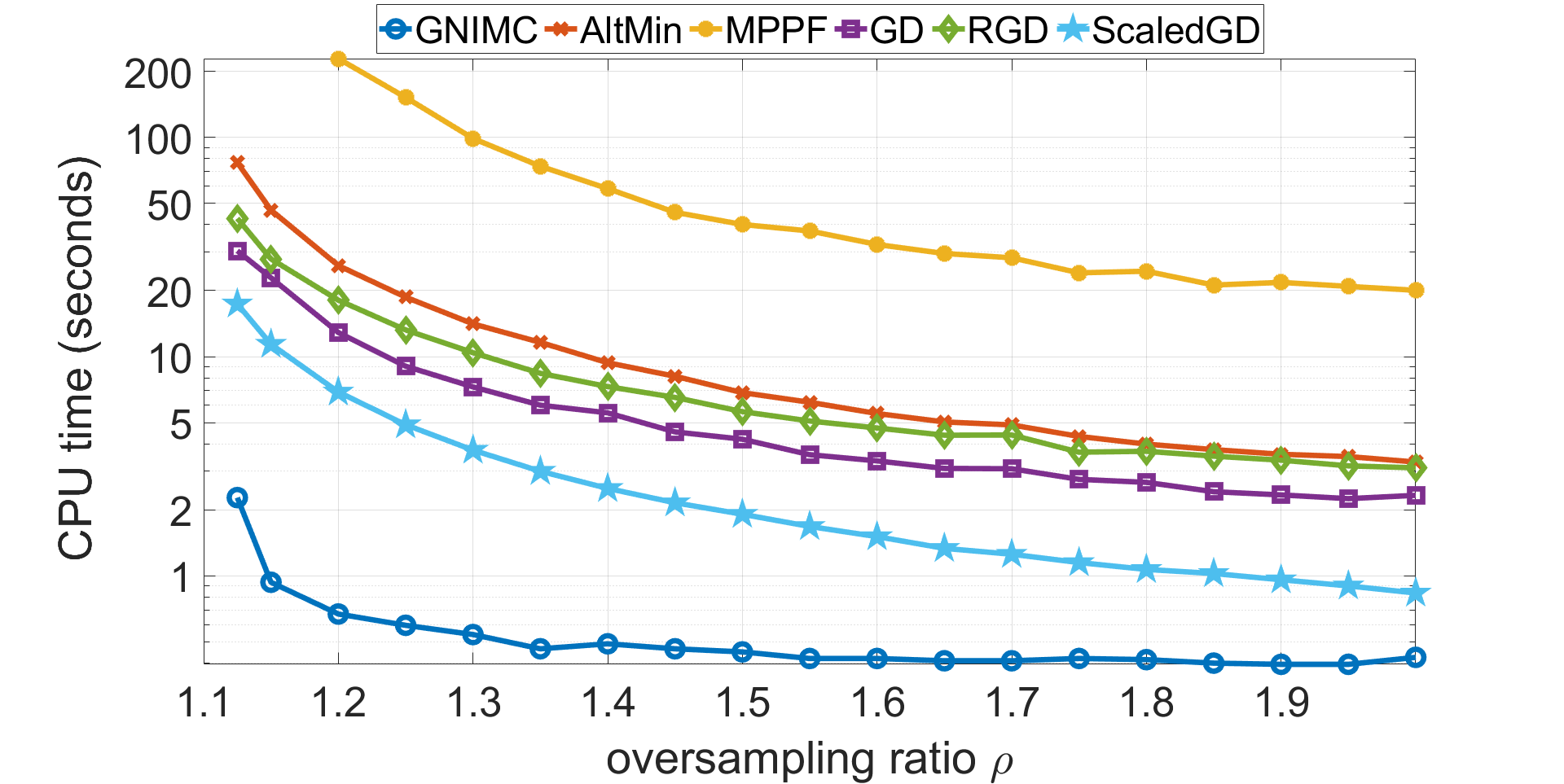}
		\label{fig:timeVsOversampling_n1000}
		}
	\subfloat{
		\includegraphics[width=0.48\linewidth]{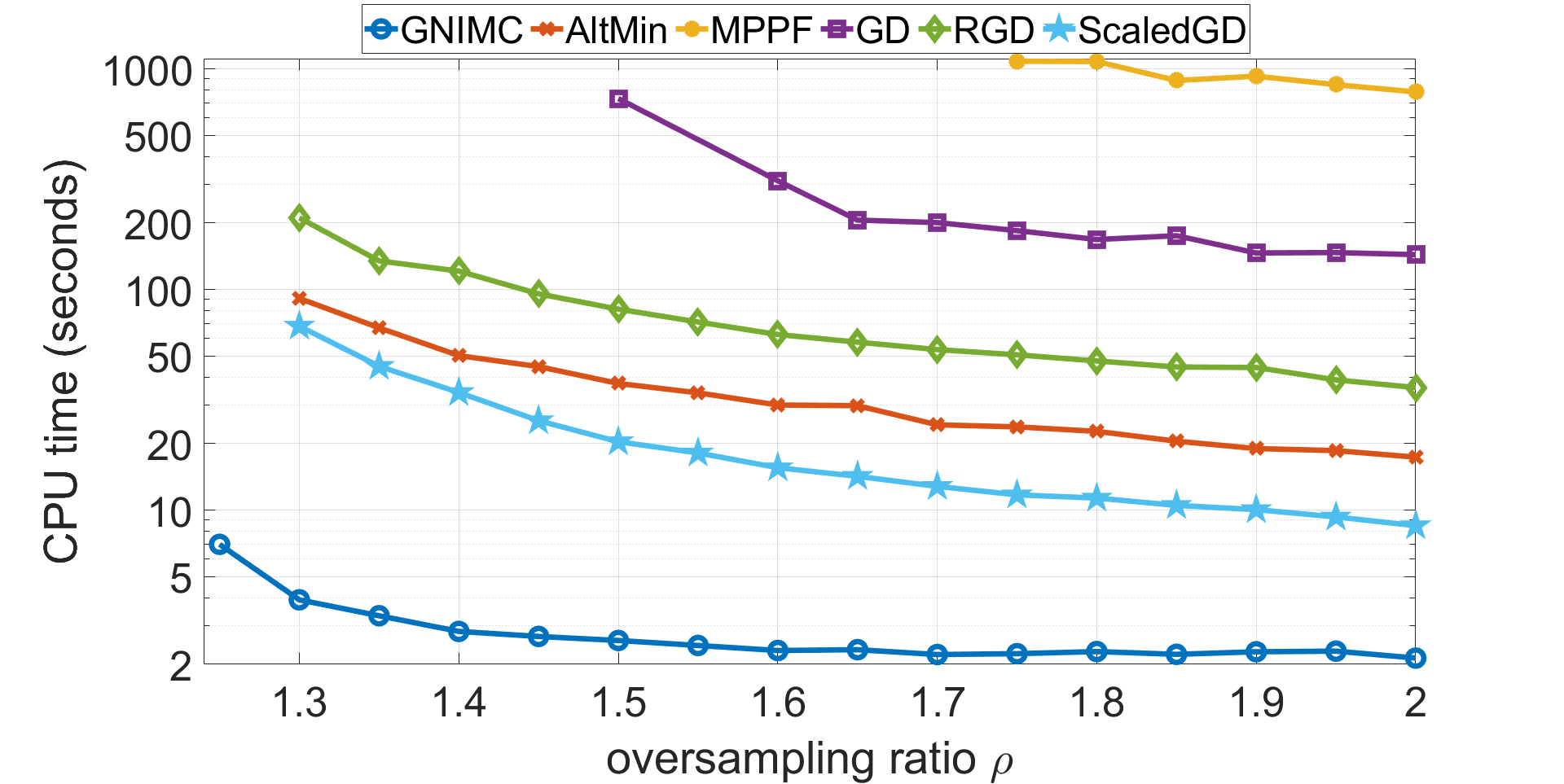}
		\label{fig:timeVsOversampling_n10000}
	}
\caption{CPU runtime till convergence as a function of the oversampling ratio for several IMC algorithms. Left panel: $n_1 = n_2 = 1000$, $d_1 = d_2 = 20$ and $r = 10$. Right panel: $n_1 = 20000$, $n_2 = 1000$, $d_1 = 100$, $d_2 = 50$ and $r = 5$. In both panels $\kappa = 10$. Each point corresponds to the median of $50$ independent realizations.}
\label{fig:timeVsOversampling}
\end{figure}

\subsection{Demonstration of the rank estimation scheme}\label{sec:experiments_rankEstimate}
In this subsection we demonstrate the accuracy of our proposed rank estimation scheme \eqref{eq:rankEstimate_scheme}. \Cref{fig:rankEstimate} compares the estimated singular gaps $\hat g_i$ with the true ones $g_i$ for a matrix of approximate rank $r=5$ and only $p = 0.1\%$ observed entries. We tested two values of $D$: $D=0$ and $D=(\sqrt{d_1d_2}/|\Omega|)^{1/2}$. The qualitative behavior depicted in the figure did not change in 50 independent realizations of the simulation. In particular, the estimated rank $\hat r = \max_i \hat g_i$ was always $5$ for both values of $D$.

The figure also demonstrates the trade-off in the choice of the value of $D$: for larger $D$, $\hat g_i$ is a more accurate estimate of $g_i$, but it also distorts the exact singular gaps $\sigma_i^*/\sigma^*_{i+1}$, especially at their tail (large values of $i$). Hence, in general, nonzero $D$ is suitable in case the rank of $X^*$ is expected to be relatively low compared to $d_1,d_2$.

\begin{figure}[t]
\centering
	\subfloat{
		\includegraphics[width=0.48\linewidth]{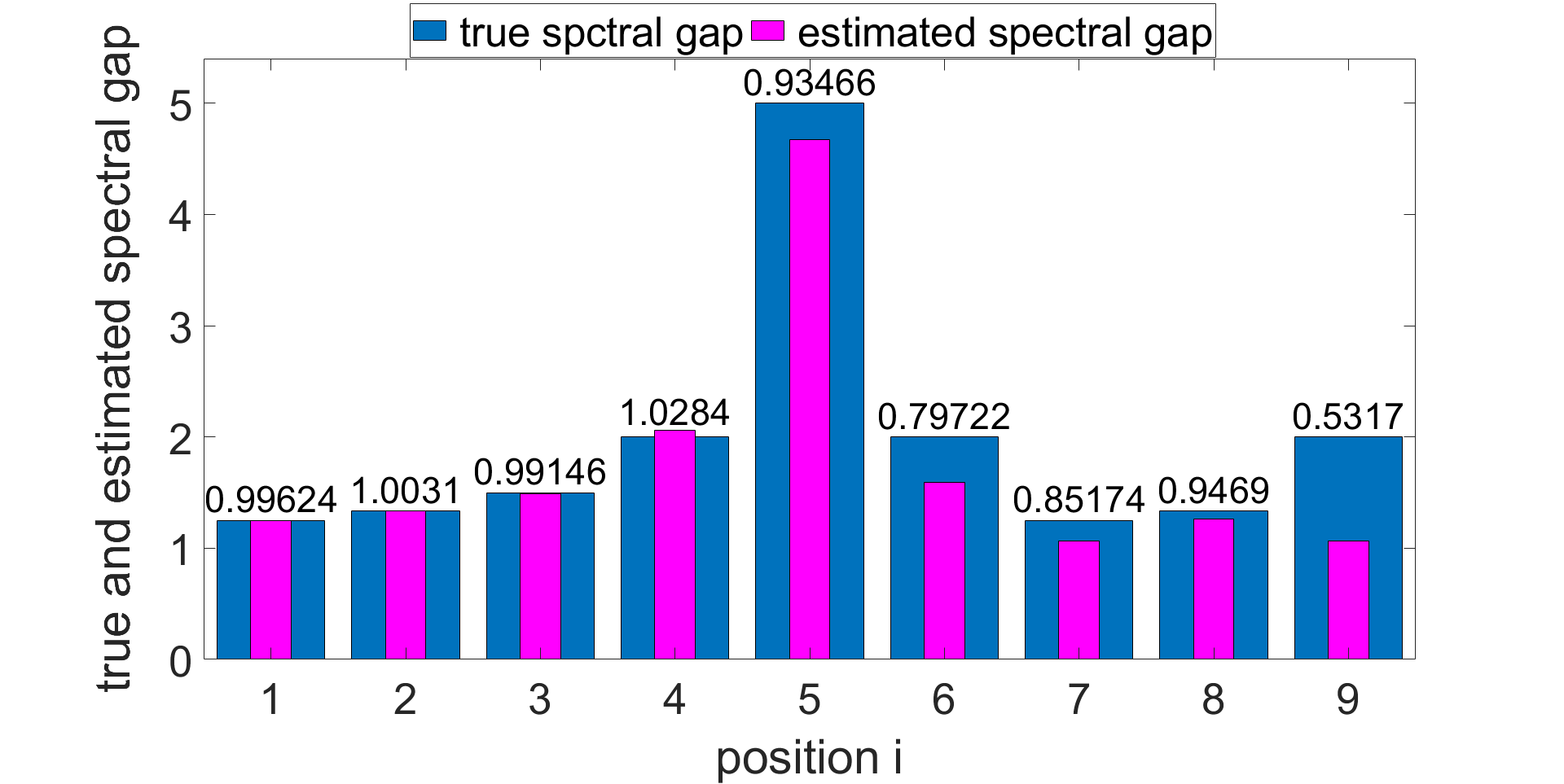}
		\label{fig:rankEstimate_withD}
	}
	\subfloat{
		 \includegraphics[width=0.48\linewidth]{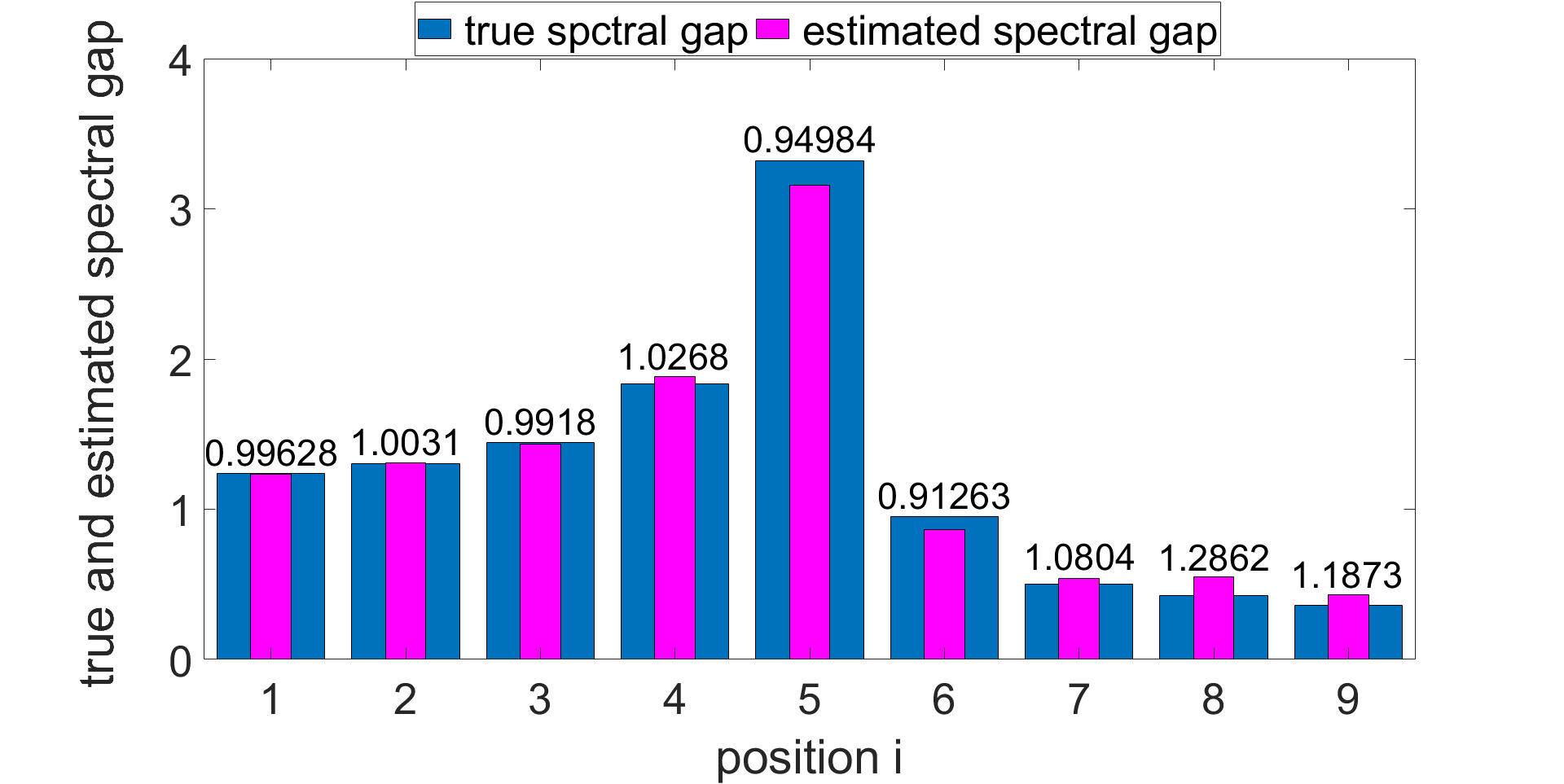}
		\label{fig:rankEstimate_noD}
 }
 \caption{The estimated spectral gaps $\hat g_i$ (inner magenta) compared to the true ones $g_i$ (outer blue) as defined in \eqref{eq:rankEstimate_scheme}, for $X^*\in \mathbb R^{30000\times 10000}$ of approximate rank $r=5$ with singular values $[5, 4, 3, 2, 1, 0.2, 0.1, 0.08, 0.06, 0.03]$, side information $d_1 = 30$, $d_2 = 20$, and sampling rate $p = 0.1\%$. The numbers above the bars indicate the ratio $\hat g_i / g_i$. Left panel: $D = 0$. Right panel: $D = (\sqrt{d_1d_2}/|\Omega|)^{1/2} \approx 0.009$.}
\label{fig:rankEstimate}
\end{figure}

\section{Summary and Discussion}\label{sec:discussion}
In this work, we presented three contributions to the IMC problem: benign optimization landscape guarantee; provable rank estimation scheme; and a simple Gauss-Newton based method, \GNIMC, to solve the IMC problem. We derived recovery guarantees for \GNIMC, and showed empirically that it is faster than several competing algorithms.
A key theoretical contribution is a proof that under relatively mild conditions, IMC satisfies an RIP, similar to the matrix sensing problem.

Interestingly, in our simulations \GNIMC recovers the matrix significantly faster than first-order methods, including a very recent one due to \cite{tong2021accelerating}.
A possible explanation is that \GNIMC makes large non-local updates, yet with the same time complexity as a single local gradient descent step.
This raises the following intriguing questions: are there other non-convex problems for which non-local methods are faster than first-order ones?
In particular, can these ideas be extended to faster training of deep neural networks?

Another interesting direction is extending our method to generalized frameworks of IMC. Important examples include recovering an unknown low rank $X^*$ which lies in some known linear subspace instead of property \eqref{eq:sideInformation} \cite{jawanpuria2018unified}, and non-linear IMC \cite{zhong2019provable}.

\nocite{langley00}

\bibliography{IMC}
\bibliographystyle{alpha}

\newpage
\appendix
\onecolumn

\textbf{Additional notation}.
In the following appendices, the Frobenius inner product between two matrices is denoted by $\braket{X,Y} = \Tr(Y^\top X)$, where $\Tr$ denotes the matrix trace.
The adjoint of an operator $\mathcal P$ is denoted by $\mathcal P^*$.
The spectral norm of an operator $\mathcal P$ that acts on matrices is defined as $\mathcal \|\mathcal P\| = \max_X \|\mathcal P(X)\|_F / \|X\|_F$.

\section{Proof of \cref{thm:IMC_RIP} (RIP for IMC)} \label{sec:proof_IMC_RIP}
In the following subsection we state and prove a novel RIP guarantee that is key to the connection between IMC and matrix sensing.
Then, in the next subsection, we use this result to prove \cref{thm:IMC_RIP}.
\subsection{An auxiliary lemma}
To present our RIP result in the context of IMC, recall the definition of the linear operator $\mathcal P_{AB}: \mathbb R^{n_1\times n_2}\to \mathbb R^{n_1\times n_2}$ which projects a matrix $X$ into the row and column spaces of the isometry matrices $A$ and $B$, respectively,
\begin{align}\label{eq:PAB}
\mathcal P_{AB}(X) = AA^\top X BB^\top.
\end{align}
Note that since $\mathcal P_{AB}$ is a projection operator, $\|\mathcal P_{AB}\| = 1$.

\begin{lemma}\label{lem:AB_RIP}
Let $A \in \mathbb R^{n_1\times d_1}$ and $B \in \mathbb R^{n_2\times d_2}$ be two isometry matrices such that $\|A\|_{2,\infty} \leq \sqrt{\mu d_1/n_1}$ and $\|B\|_{2,\infty} \leq \sqrt{\mu d_2/n_2}$.
Let $\delta \in [0,1)$, and assume $\Omega \subseteq [n_1]\times [n_2]$ is uniformly sampled with $|\Omega| \equiv n_1n_2p \geq (8/\delta^2) \mu^2 d_1 d_2 \log n$ where $n = \max\{n_1,n_2\}$.
Then w.p.~at least $1-2n^{-2}$,
\begin{align}\label{eq:PAB_inequality}
\|\tfrac 1p \mathcal P_{AB} \mathcal P_\Omega \mathcal P_{AB} - \mathcal P_{AB} \| \leq \delta.
\end{align}
\end{lemma}

The numerical factor $8$ in the bound on the sample complexity $|\Omega|$ of \cref{lem:AB_RIP} can be replaced by any other scalar $\beta$ strictly greater than $8/3$, resulting in a modified probability guarantee $1 - 2n^{1-3\beta/8}$. We remark that $8/3$ is strict for our proof technique, which builds upon Recht's work \cite{recht2011simpler}, but can be improved by a more careful analysis, see the discussion after Proposition~5 in \cite{recht2011simpler}.

The proof of \cref{lem:AB_RIP} uses the following matrix Bernstein inequality \cite[Theorem~1.6]{tropp2012user}.
\begin{lemma}\label{lem:Bernstein_inequality}
Consider a finite set $\{Z_k\}$ of independent, random matrices with dimensions $n_1\times n_2$.
Assume that each random matrix satisfies
\begin{align}\label{eq:Bernstein_inequality_condition}
\mathbb E[Z_k] = 0 \quad\mbox{and}\quad
\|Z_k\| \leq R \quad\mbox{almost surely}.
\end{align}
Define $\sigma^2 = \max\{\|\sum_k \mathbb E[Z_k Z_k^\top]\|, \|\sum_k \mathbb E[Z_k^\top Z_k]\|\}$.
Then, for all $t\geq 0$,
\begin{align}\label{eq:Bernstein_inequality}
\mathbb P\left[\left\|\sum_k Z_k \right\| \geq t\right] \leq (n_1+n_2) \exp\left(-\frac{t^2/2}{\sigma^2 + Rt/3}\right).
\end{align}
\end{lemma}

\begin{proof}[Proof of \cref{lem:AB_RIP}]
The lemma assumes that $\Omega$ is uniformly sampled from the set of all collections of $m \equiv n_1n_2p$ entries of $[n_1]\times [n_2]$.
Following \cite{recht2011simpler}, in the following proof we assume instead a different probabilistic model: sampling with replacement. Let $\Omega' = \{(i_k, j_k)\}_{k=1}^m$ be a collection of $m$ elements, each i.i.d.~from the uniform distribution over $[n_1]\times [n_2]$. Define also the corresponding operator
\begin{align}\label{eq:R_Omega'}
\mathcal R_{\Omega'}(X) = \sum_{k=1}^m \braket{e_{i_k} e_{j_k}^\top, X} e_{i_k} e_{j_k}^\top.
\end{align}
In contrast to $\mathcal P_\Omega$, the operator $\mathcal R_{\Omega'}$ is in general not a projection operator, since a pair of indices $(i,j)$ may have been sampled more than once.
In the following, rather than \eqref{eq:PAB_inequality}, we prove the following modified inequality that involves $\mathcal R_{\Omega'}$ in place of $\mathcal P_\Omega$,
\begin{align}\label{eq:PAB_inequality_modified}
\tfrac 1p \|\mathcal P_{AB} \mathcal R_{\Omega'} \mathcal P_{AB} - p \mathcal P_{AB} \| \leq \delta .
\end{align}
This inequality implies the original \eqref{eq:PAB_inequality},
as $\mathcal R_{\Omega'}(X)$ reveals in general less information on $X$ than $\mathcal P_\Omega(X)$ does due to possible duplicates in $\Omega'$; see the proof of Proposition~3 in \cite{recht2011simpler} for a rigorous formulation of this argument.

Since the elements of $\Omega'$ are uniformly sampled from the set $[n_1]\times [n_2]$ and $|\Omega'| = m \equiv pn_1n_2$, the expectation value of $\mathcal R_{\Omega'}$ over the random set $\Omega'$ is $p$ times the identity operator. Hence,
\begin{align}
\mathbb E[\mathcal P_{AB} \mathcal R_{\Omega'} \mathcal P_{AB}] = \mathcal P_{AB} \mathbb E[\mathcal R_{\Omega'}] \mathcal P_{AB} = p \mathcal P_{AB}^2 = p \mathcal P_{AB} ,
\end{align}
where $\mathcal P_{AB}$ is defined in \eqref{eq:PAB}.
We thus conclude that \eqref{eq:PAB_inequality_modified} is simply a concentration inequality, which we shall prove using \cref{lem:Bernstein_inequality}.

Let $X \in \mathbb R^{n_1\times n_2}$, and decompose it as $X = \sum_{i,j} \braket{X, e_ie_j^\top} e_ie_j^\top$.
For future use, we define the linear operator $\mathcal T_{ij}: \mathbb R^{n_1\times n_2}\to \mathbb R^{n_1\times n_2}$ as
\begin{align}\label{eq:Tij}
\mathcal T_{ij}(X) = \braket{X, \mathcal P_{AB}(e_ie_j^\top)} \mathcal P_{AB}(e_ie_j^\top) = \braket{\mathcal P_{AB}(X), e_ie_j^\top} \mathcal P_{AB}(e_ie_j^\top),
\end{align}
and present some related equalities.
By standard properties of the trace operator,
\begin{align}\label{eq:partial_sum_of_Tij}
\mathcal P_{AB} \mathcal R_{\Omega'} \mathcal P_{AB} = \sum_{k=1}^m \mathcal T_{i_k j_k}.
\end{align}
Hence, taking the expectation over $(i,j)$ uniformly sampled from $[n_1]\times [n_2]$ gives that
\begin{align}\label{eq:Tij_expectation}
\mathbb E[ \mathcal T_{i j} ] &= \frac{1}{m} \mathbb E \left[ \sum_{k=1}^m \mathcal T_{i_k j_k} \right] = \frac{1}{pn_1n_2} \mathbb E[ \mathcal P_{AB} \mathcal R_{\Omega'} \mathcal P_{AB} ] = \frac{1}{n_1n_2} \mathcal P_{AB}.
\end{align}
In addition, by the definition \eqref{eq:Tij} of $\mathcal T_{ij}$ and the fact that $\mathcal P_{AB}$ is a projection,
\begin{align}\label{eq:PAB_Tij}
\mathcal P_{AB} \mathcal T_{ij} = \mathcal T_{ij} \mathcal P_{AB} = \mathcal T_{ij} .
\end{align}
Finally, by inserting \eqref{eq:partial_sum_of_Tij} into inequality \eqref{eq:PAB_inequality_modified}, we obtain that our goal is to bound $\|\sum_{k=1}^m \mathcal T_{i_k j_k} - p \mathcal P_{AB}\| = \|\sum_{k=1}^m \mathcal D_{i_k j_k}\|$, where the operator $\mathcal D_{ij}: \mathbb R^{n_1\times n_2}\to \mathbb R^{n_1\times n_2}$ is given by
\begin{align*}
\mathcal D_{ij} = \mathcal T_{ij} - \frac{p}{m} \mathcal P_{AB} = \mathcal T_{ij} - \frac{1}{n_1n_2} \mathcal P_{AB} .
\end{align*}
By \eqref{eq:Tij_expectation}, $\mathbb E[\mathcal D_{ij}] = 0$.
To employ \cref{lem:Bernstein_inequality} to the set $\{\mathcal D_{i_k,j_k}\}_{k=1}^m$, we first need to (i) find a scalar $R$ such that $\|\mathcal D_{ij}\| \leq R$ almost surely, and (ii) bound $\max\{\| \sum_{k=1}^m \mathbb E[\mathcal D_{i_k j_k} \mathcal D_{i_k j_k}^*] \|,$ $\| \sum_{k=1}^m \mathbb E[\mathcal D_{i_k j_k}^* \mathcal D_{i_k j_k}] \|\} = \| \sum_{k=1}^m \mathbb E[\mathcal D_{i_k j_k}^2] \|$, where the equality follows since $\mathcal D_{ij}$ is self-adjoint w.r.t.~the Frobenius inner product.

We begin with bounding $\|\mathcal D_{ij}\| \equiv \max_X \|\mathcal D_{ij}(X)\|_F/\|X\|_F$.
Recall that if $X$ and $Y$ are positive semidefinite matrices, then $\|X-Y\|_2 \leq \max\{\|X\|_2,\|Y\|_2\}$.
Since any operator can be represented by a matrix, a similar result holds for operators with the spectral norm.
As both $\mathcal T_{ij}$ and $\mathcal P_{AB}$ are positive semidefinite and $\mathcal P_{AB}$ is a projection,
we have
\begin{align}
\| \mathcal D_{ij} \| \leq \max\{\|\mathcal T_{ij}\|, \frac{1}{n_1n_2} \|\mathcal P_{AB}\| \} = \max\{\|\mathcal T_{ij}\|, \frac{1}{n_1n_2} \}. \label{eq:Tij_PAB_tempBound}
\end{align}
Let us bound $\|\mathcal T_{ij}\|$. By the Cauchy-Schwarz inequality,
\begin{align*}
\|\mathcal T_{ij}(X)\| &= |\braket{X, \mathcal P_{AB}(e_ie_j^\top)}|\cdot \|\mathcal P_{AB}(e_ie_j^\top)\|_F
\leq \|\mathcal P_{AB}(e_ie_j^\top)\|_F^2 \|X\|_F.
\end{align*}
Inserting the definition of $\mathcal P_{AB}$ \eqref{eq:PAB}, the spectral norm of $\mathcal T_{ij}$ is bounded as
\begin{align}\label{eq:Tij_norm}
\|\mathcal T_{ij}\| &\leq \|\mathcal P_{AB}(e_i e_j^\top) \|_F^2 = \|AA^\top e_i e_j^\top B B^\top \|_F^2 \stackrel{(a)}{\leq} \|AA^\top e_i\|^2 \|B B^\top e_j\|^2 \nonumber \\
&\stackrel{(b)}{=} \|A^\top e_i\|^2 \|B^\top e_j\|^2 \leq \|A\|^2_{2,\infty} \|B\|^2_{2,\infty} \stackrel{(c)}{\leq} \frac{\mu^2 d_1 d_2}{n_1 n_2},
\end{align}
where (a) follows from the Cauchy-Schwarz inequality, (b) from the isometry assumption, and (c) from the assumed bound on the row norms of $A$ and $B$. Plugging \eqref{eq:Tij_norm} into \eqref{eq:Tij_PAB_tempBound} yields
\begin{align}
\| \mathcal D_{ij} \| \leq \max\{\frac{\mu^2 d_1 d_2}{n_1 n_2}, \frac{1}{n_1n_2} \} = \frac{\mu^2 d_1 d_2}{n_1 n_2} \equiv R, \label{eq:Tij_PAB_bound}
\end{align}
where the equality follows since $\mu \geq 1$ by the definition of incoherence (\cref{def:incoherence}).
Next, we bound $\|\sum_{k=1}^m \mathbb E[\mathcal D_{i_k j_k}^2]\|$.
Combining \eqref{eq:PAB_Tij}, \eqref{eq:Tij_expectation} and the fact that both $\mathcal T_{ij}^2$ and $\mathcal P_{AB}$ are positive semidefinite yields
\begin{align*}
\| \mathbb E [\mathcal D_{i j}^2 ] \| &= \| \mathbb E[ \mathcal T_{i j}^2 - \frac{2}{n_1n_2} \mathcal T_{i j} + \frac{1}{n_1^2n_2^2} \mathcal P_{AB} ] \| = \| \mathbb E[\mathcal T_{i j}^2] - \frac{1}{n_1^2n_2^2} \mathcal P_{AB} \| \\
&\leq \max\{\|\mathbb E[\mathcal T_{i j}^2]\|, \frac{1}{n_1^2n_2^2} \|\mathcal P_{AB}\|\} 
= \max\{\|\mathbb E[\mathcal T_{i j}^2]\|, \frac{1}{n_1^2n_2^2} \} .
\end{align*}
Let us bound $\|\mathbb E[\mathcal T_{i j}^2]\|$.
Since $\mathcal T_{ij}$ is positive semidefinite, we have $\mathcal T_{ij}^2 \preccurlyeq \|\mathcal T_{ij}\| \mathcal T_{ij}$. 
Thus $\mathbb E[\mathcal T_{i j}^2]  \preccurlyeq \mathbb E[\|\mathcal T_{i j}\| \mathcal T_{i j}] \preccurlyeq \frac{\mu^2 d_1 d_2}{n_1 n_2} \mathbb E[\mathcal T_{i j}]$, where the last inequality follows from the deterministic bound \eqref{eq:Tij_norm}. Together with \eqref{eq:Tij_expectation} this implies 
\begin{align*}
\| \mathbb E[ \mathcal T_{i j}^2 ] \| &\leq \frac{\mu^2 d_1 d_2}{n_1 n_2} \| \mathbb E [\mathcal T_{i j}] \| = \frac{\mu^2 d_1 d_2}{n_1^2 n_2^2} \|\mathcal P_{AB} \| = \frac{\mu^2 d_1 d_2}{n_1^2 n_2^2} .
\end{align*}
We thus obtain the bound
\begin{align*}
\| \sum_{k=1}^m \mathbb E [ \mathcal D_{i_k j_k}^2 ] \| = m\cdot \| \mathbb E [ \mathcal D_{i j}^2 ] \| \leq m \frac{\mu^2 d_1 d_2}{n_1^2 n_2^2} = \frac{p\mu^2 d_1d_2}{n_1n_2} \equiv \sigma^2.
\end{align*}
Plugging this together with the bound $\|\mathcal D_{ij}\| \leq R$ in \eqref{eq:Tij_PAB_bound} into \cref{lem:Bernstein_inequality} yields
\begin{align*}
\mathbb P\left[ \left\| \sum_{k=1}^m \mathcal D_{i_k j_k} \right\| > p\delta \right] &\leq (n_1+n_2) \exp\left( -\frac{p^2 \delta^2/2}{  \frac{p\mu^2 d_1d_2}{n_1n_2} + \frac{\mu^2 d_1 d_2}{n_1 n_2} p\delta/3 } \right) \leq 2n \exp\left( -\frac{3 \delta^2 m}{8\mu^2 d_1 d_2} \right) .
\end{align*}
Assuming that $m \geq (8/\delta^2) \mu^2 d_1d_2 \log n$ gives
\begin{align*}
\mathbb P\left[ \left\| \sum_{k=1}^m \mathcal D_{i_k j_k} \right\| > p\delta \right] &\leq 2n e^{-3\log n} = 2n^{-2} .
\end{align*}
This completes the proof of \eqref{eq:PAB_inequality_modified}, and thus of \eqref{eq:PAB_inequality}.
\end{proof}

\subsection{Proof of \cref{thm:IMC_RIP}}
Let $M \in \mathbb R^{d_1\times d_2}$, and denote $X = AMB^\top$.
By definition \eqref{eq:sensingOperator_IMC} of $\mathcal A$,
\begin{align}\label{eq:X_Omega_F2}
\frac 1p \|\mathcal P_\Omega(X)\|_F^2 = \frac 1p \|\mathcal P_\Omega(AMB^\top)\|_F^2 = \|\mathcal A(M)\|^2.
\end{align}
Next, observe that $\mathcal P_{AB}(X) = AA^\top AMB^\top BB^\top = AMB^\top = X$. Hence
\begin{align*}
\|\mathcal P_{\Omega}(X)\|_F^2 &= \braket{\mathcal P_{\Omega}(X), \mathcal P_{\Omega}(X)} = \braket{X, \mathcal P_{\Omega}(X)} = \braket{X, pX} + \braket{X, \mathcal P_{\Omega}(X) - pX} \\
&= p\|X\|_F^2 + \braket{\mathcal P_{AB}(X), \mathcal P_{\Omega}\mathcal P_{AB}(X) - p\mathcal P_{AB}(X)} \\
&= p \|X\|_F^2 + \braket{X, \mathcal P_{AB}\mathcal P_{\Omega} \mathcal P_{AB}(X) - p\mathcal P_{AB}(X) } .
\end{align*}
Applying the Cauchy-Schwarz inequality and \eqref{eq:PAB_inequality} of \cref{lem:AB_RIP} yields
\begin{align*}
\left| \|\mathcal P_{\Omega}(X)\|_F^2 - p \|X\|_F^2 \right| &= |\braket{X, \mathcal P_{AB}\mathcal P_{\Omega} \mathcal P_{AB}(X) - p\mathcal P_{AB}(X) }| \\
&\leq \|X\|_F \|\mathcal P_{AB}\mathcal P_{\Omega} \mathcal P_{AB}(X) - p\mathcal P_{AB}(X)\|_F \leq p\delta \|X\|_F^2 .
\end{align*}
Hence
\begin{align}\label{eq:RIP_X_temp}
(1-\delta) \|X\|_F^2 \leq \frac 1p \|\mathcal P_\Omega(X)\|_F^2 \leq (1+\delta) \|X\|_F^2 .
\end{align}
Since $A,B$ are isometries, $\|X\|_F = \|AMB^\top\|_F = \|M\|_F$.
Plugging this together with \eqref{eq:X_Omega_F2} into \eqref{eq:RIP_X_temp} yields the RIP \eqref{eq:RIP}.
\qed

In the following remark, we extend the connection between IMC and matrix sensing (MS) to another setting of the two problems, where the goal is to find the minimal rank matrix that agrees with the observations. 

\begin{remark}\label{remark:IMC_RIP}
An alternative setting of IMC, which does not assume a known rank but does assume noise-free observations, is to find a matrix with the lowest possible rank that is consistent with the data,
\begin{align}
\tag{IMC*} \label{eq:IMC_unknownRank}
\min_M \text{ rank}(M) \quad \text{s.t. } \mathcal P_\Omega(AMB^\top) = \mathcal P_\Omega(X^*) .
\end{align}
The analogous setting of MS is
\begin{align}
\tag{MS*}\label{eq:MS_unknownRank}
\min_M \text{ rank}(M) \quad &\text{s.t. } \mathcal A(M) = \mathcal A(M^*) .
\end{align}
\end{remark}
With the sensing operator $\mathcal A$ defined in \eqref{eq:sensingOperator_IMC}, \eqref{eq:IMC_unknownRank} is in the form of \eqref{eq:MS_unknownRank}. Since  this sensing operator satisfies the RIP under certain conditions as guaranteed by \cref{thm:IMC_RIP}, the connection between IMC and MS holds in this setting as well.

\section{Proof of \cref{thm:rankEstimate} (Rank Estimation)} \label{sec:proof_rankEstimate}
The proof of the theorem is based on the following lemma, which employs \cref{lem:AB_RIP} to bound the difference between the singular values of $X^*$ and those of $\hat X = \mathcal P_{AB}(Y)/p$.
\begin{lemma}\label{lem:rankEstimate}
Let $X^* \in \mathbb R^{n_1\times n_2}$ be a matrix which satisfies \eqref{eq:sideInformation} with $\mu$-incoherent matrices $A,B$.
Let $\delta, \epsilon$ and $\Omega$ be defined as in \cref{thm:rankEstimate} with constant $c < 1/2$.
Then w.p.~at least $1-2n^{-2}$,
\begin{align}\label{eq:sv_diff}
|\hat \sigma_i - \sigma_i^*| \leq 2c\delta, \quad \forall i,
\end{align}
where $\sigma_i^* = \sigma_i(X^*)$.
\end{lemma}

\begin{proof}
Since $X^*$ satisfies the side information property \eqref{eq:sideInformation}, we have $\mathcal P_{AB}(X^*) = X^*$. Hence
\begin{align*}
\hat X = \frac 1p \mathcal P_{AB} \mathcal P_\Omega(X^* + \mathcal E) = \frac 1p \mathcal P_{AB} \mathcal P_\Omega \mathcal P_{AB}(X^*) + \frac 1p \mathcal P_{AB} \mathcal P_\Omega(\mathcal E) .
\end{align*}
Using $\mathcal P_{AB}(X^*) = X^*$ again, we get
\begin{align*}
\hat X - X^* = \left( \frac 1p \mathcal P_{AB} \mathcal P_\Omega \mathcal P_{AB} - \mathcal P_{AB}\right)(X^*) + \frac 1p \mathcal P_{AB} \mathcal P_\Omega(\mathcal E).
\end{align*}
Let $\delta' = c\delta/\|X^*\|_F$. By definition, $\delta \leq \sigma_2^* + D\sigma_1^* < 2\sigma_1^*$. Hence $\delta' < 1$ for $c < 1/2$.
Invoking \cref{lem:AB_RIP} with $|\Omega| \geq 8 \mu^2 d_1 d_1 \log(n) \|X^*\|_F^2 / (c\delta)^2 = 8 \mu^2 d_1 d_1 \log(n) / {\delta'}^2$ and using the condition $\epsilon \leq c\delta$ imply
\begin{align}\label{eq:rankEstimate_matrixDiffNorm}
\|\hat X - X^* \|_F &\leq \left\|\left(\frac 1p \mathcal P_{AB} \mathcal P_\Omega \mathcal P_{AB} - \mathcal P_{AB}\right)(X^*) \right\|_F + \frac 1p \|\mathcal P_{AB} \mathcal P_\Omega(\mathcal E)\|_F \leq \delta' \|X^*\|_F + \epsilon \nonumber \\
&\leq 2c \delta.
\end{align}
Hence also $\|\hat X - X^*\|_F \leq 2c \delta$. \Cref{eq:sv_diff} of the lemma follows by Weyl's inequality.
\end{proof}

\begin{proof}[Proof of \cref{thm:rankEstimate}]
Denote $\hat g_i = g_i(\hat X)$ and $g_i^* = g_i(X^*)$. We need to show that $\argmax_i \hat g_i = r$. Invoking \cref{lem:rankEstimate} implies
\begin{align}\label{eq:hat_gr_temp}
\hat g_r &= \frac{\hat \sigma_r}{\hat \sigma_{r+1} + D\hat \sigma_1 \sqrt{r}} \geq \frac{\sigma_r^* - 2c\delta}{\sigma_r^* + 2c\delta + D(\sigma_1^* + 2c\delta) \sqrt{r}}.
\end{align}
By the definition of $\delta$, we have that $\delta \leq (1+D)\sigma_1^* < 2\sigma_1^*$ and also $\delta \leq \sigma_r^* + D \sigma_1^* \sqrt{r}$. Plugging this into \eqref{eq:hat_gr_temp} yields
\begin{align*}
\hat g_r &\geq \frac{\sigma_r^* - 2c\delta}{\sigma_r^* + 2c\delta + D(\sigma_1^* + 4c\sigma_1^*) \sqrt{r}}
\geq \frac{\sigma_r^* - 2c(\sigma_r^* + D\sigma_1^*\sqrt{r})}{(1 + 2c)(\sigma_r^* + D \sigma_1^* \sqrt{r}) + 4cD\sigma_1^* \sqrt{r}} \\
&\geq \frac{\sigma_r^* - 2c(\sigma_r^* + D\sigma_1^* \sqrt{r})}{(1 + 6c)(\sigma_r^* + D\sigma_1^* \sqrt{r})} = \frac{1}{1+6c} g_r^* - \frac{2c}{1+6c}.
\end{align*}
Next, let $i\neq r$. Since $\delta \leq \sigma_{i+1}^* + D \sigma_1^* \sqrt{i}$, we similarly have
\begin{align*}
\hat g_i &= \frac{\hat \sigma_i}{\hat \sigma_{i+1} + D\hat \sigma_1 \sqrt{i}} \leq \frac{\sigma_i^* + 2c\delta}{\sigma_{i+1}^* - 2c\delta + D(\sigma_1^* - 2c\delta) \sqrt{i}}
\leq \frac{\sigma_i^* + 2c(\sigma^*_{i+1} + D\sigma_1^*\sqrt{i})}{(1-2c)(\sigma_{i+1}^* + D\sigma_1^*\sqrt{i}) - 4cD\sigma^*_{i+1}\sqrt{i}} \\
&\leq \frac{\sigma_i^* + 2c(\sigma_i^* + D\sigma_1^*\sqrt{i})}{(1-6c) (\sigma_{i+1}^* D\sigma_1^*\sqrt{i})}
= \frac{1}{1-6c} g_i^* + \frac{2c}{1-6c} .
\end{align*}
By assumption, $g_r^* \geq \min\{(11/10) g_i^*, 1/10\}$. We thus obtain that $\hat g_r > \hat g_i$ for a sufficiently small constant $c$. Hence $\hat r = \argmax_i \hat g_i = r$, as required.
\end{proof}

\section{Proof of \cref{thm:GNIMC_guarantee,thm:GNIMC_guarantee_noisy,thm:IMC_landscape}}\label{sec:proof_RIP_consequences}
Our proof of \cref{thm:IMC_landscape} follows by combining \cref{thm:IMC_RIP} with a general result due to \cite{li2020global}.
Consider the following general low rank optimization problem,
\begin{align}\label{eq:general_objective}
\min_{M\in \mathbb R^{d_1\times d_2}} f(M), \quad \text{s.t. rank}(M) \leq r.
\end{align}
By incorporating the rank constraint into the objective function, we obtain the factorized problem
\begin{align}\label{eq:general_objective_factorized}
\min_{U\in \mathbb R^{d_1\times r}, V\in \mathbb R^{d_2\times r}} g(U,V) \equiv f(UV^\top) .
\end{align}
The following result
provides a sufficient condition on $f(M)$ such that $g(U,V)$ has no bad local minima. The condition is on the bilinear form of the Hessian of $f(M)$, defined as $\nabla^2 f(M)[N,N] = \sum_{i,j,k,l} \frac{\partial ^2 f(M)}{\partial M_{ij} \partial M_{kl}} N_{ij} N_{kl}$.

\begin{lemma}\label{lem:Li_theorem3.1}
Let $\alpha, \beta$ be two positive constants that satisfy $\beta/\alpha \leq 3/2$.
Assume that $f$ satisfies
\begin{align}\label{eq:Li_theorem3.1_condition}
\alpha \|N\|_F^2 \leq \nabla^2 f(M)[N,N] \leq \beta \|N\|_F^2
\end{align}
for all $M, N\in \mathbb R^{d_1\times d_2}$.
If $f(M)$ has a critical point $M^*$ with $\text{rank}(M^*) \leq r$,
then any critical point $(U,V)$ of $g(U,V)$ in \eqref{eq:general_objective_factorized} is either a global minimum with $UV^\top = M^*$ or a strict saddle point.
\end{lemma}

\Cref{lem:Li_theorem3.1} is similar to Theorem~III.1 in \cite{li2020global}, with one difference: in their Theorem~III.1, it is sufficient that condition \eqref{eq:Li_theorem3.1_condition} holds only for matrices $M,N$ of rank at most $r_1, r_2$, respectively, with $r_1 = \min\{2r, d_1, d_2\}$ and $r_2 = \min\{4r, d_1, d_2\}$. In fact, since \cite{li2020global} assume $r \ll \min\{d_1,d_2\}$ throughout their work, their Theorem~III.1 is phrased with $r_1 = 2r$ and $r_2 = 4r$; however, it is straightforward to verify that in the general case, in which the rank of $M,N$ is bounded by $\min\{d_1,d_2\}$, the theorem holds with $r_1 = \min\{2r, d_1, d_2\}$ and $r_2 = \min\{4r, d_1, d_2\}$.
This condition is known as $(r_1,r_2)$-restricted strongly convex smoothness. Our condition is stronger, as it requires \eqref{eq:Li_theorem3.1_condition} to hold for all $M,N$, and thus implied by their Theorem~III.1.

\begin{proof}[Proof of \cref{thm:IMC_landscape}]
As discussed in the main text, the IMC problem can be written as a matrix sensing problem with the objective $f(M) = \|\mathcal A(M) - b\|^2$, the sensing operator $\mathcal A$ given in \eqref{eq:sensingOperator_IMC}, and $b = \mathcal A(M^*) + \text{Vec}_\Omega(\mathcal E)/\sqrt p$. Furthermore, by \cref{thm:IMC_RIP}, for the assumed $|\Omega|$, the operator $\mathcal A$ satisfies a $\min\{d_1, d_2\}$-RIP \eqref{eq:RIP} with a constant $\delta \leq 1/5$.
Note that the $\min\{d_1,d_2\}$-RIP of $\mathcal A:\mathbb R^{d_1\times d_2}\to \mathbb R^{d_1\times d_2}$ in fact means that \eqref{eq:RIP} holds for any $d_1\times d_2$ matrix, since the rank of any such matrix is bounded by $\min\{d_1,d_2\}$.

Next, for any $M,N \in \mathbb R^{d_1\times d_2}$, we have $\nabla f(M) = \mathcal A^* (\mathcal A(M) - b)$ and $\nabla^2 f(M)[N,N] = \|\mathcal A(N)\|^2$ \cite[Section~C.1]{zhu2018global}.
Plugging the last equality into the RIP \eqref{eq:RIP} of the sensing operator $\mathcal A$ yields
\begin{align*}
(1-\delta) \|N\|_F^2 \leq \nabla^2 f(M)[N,N] \leq (1+\delta) \|N\|_F^2.
\end{align*}
Let $\alpha = 1-\delta$ and $\beta = 1+\delta$.
Then $f$ satisfies \cref{eq:Li_theorem3.1_condition} with the constants $\alpha,\beta$. Further, since $\delta \leq 1/5$, we have $\beta/\alpha \leq 3/2$.
The corollary thus follows by \cref{lem:Li_theorem3.1}.
\end{proof}

Finally, the proof of \cref{thm:GNIMC_guarantee,thm:GNIMC_guarantee_noisy} is straightforward thanks to our \cref{thm:IMC_RIP}.

\begin{proof}[Proof of \cref{thm:GNIMC_guarantee,thm:GNIMC_guarantee_noisy}]
By \cref{thm:IMC_RIP}, \eqref{eq:IMC} is a special case of \eqref{eq:MS} where the sensing operator $\mathcal A$ satisfies a rank $\min\{d_1,d_2\}$-RIP with a constant $\delta \leq 1/2$. \Cref{thm:GNIMC_guarantee,thm:GNIMC_guarantee_noisy} thus follow from the MS recovery guarantees for \texttt{GNMR} \cite[Theorems~3.3-3.4]{zilber2022gnmr}.
\end{proof}

\section{Computational Complexity Analysis}\label{sec:time_complexity}
In \cref{sec:GNIMC} of the main text we briefly mentioned a way to use QR decompositions in order to efficiently find the minimal norm solution to the least squares problem \eqref{eq:GNIMC_LSQR}.
In the following subsection we describe the full procedure in detail.
Then, in the next subsection, we prove \cref{proposition:time_complexity} on the corresponding computational complexity.
In both subsections we use the following simple result.
\begin{lemma}\label{lem:full_rank}
Assume the conditions of \cref{proposition:time_complexity}.
Then w.p.~at least $1-2n^{-2}$, the factor matrices $U_t, V_t$ of the iterates of \GNIMC (\cref{alg:GNIMC}) have full column rank for all $t = 0, 1, ...$.
\end{lemma}
\begin{proof}
We prove that if
\begin{align}\label{eq:UtVt_accuracy}
\|AU_tV_t^\top B^\top - X^*\|_F < \sigma_r^*,
\end{align}
then $U_t$ and $V_t$ are full column rank.
The lemma follows since \eqref{eq:UtVt_accuracy} holds at $t=0$ by assumption \eqref{eq:initialization_accuracy} with $c>1$, and at any $t>0$ w.p.~at least $1-2n^{-2}$ by the contraction principle \eqref{eq:GNIMC_guarantee}.

By combining Weyl's inequality and \eqref{eq:UtVt_accuracy},
\begin{align*}
|\sigma_r(AU_t V_t^\top B^\top) - \sigma_r^*| \leq \|AU_t V_t^\top B^\top - X^*\|_2 \leq \|AU_t V_t^\top B^\top - X^*\|_F < \sigma_r^*.
\end{align*}
Since $A$ and $B$ are isometries, the above inequality implies that $|\sigma_r(U_t V_t^\top) - \sigma_r^*| < \sigma_r^*$. Hence
\begin{align*}
0 < \sigma_r(U_tV_t^\top) \leq \min\{\sigma_r(U_t) \|V_t\|_2, \sigma_r(V_t) \|U_t\|_2\},
\end{align*}
which implies that both $\sigma_r(U_t)$ and $\sigma_r(V_t)$ are strictly positive, namely $U_t, V_t$ have full column rank.
\end{proof}

\subsection{A computationally efficient way to find the minimal norm solution to \eqref{eq:GNIMC_LSQR}}\label{sec:efficient_procedure}
At iteration $t$ of \GNIMC (\cref{alg:GNIMC}), our goal is to efficiently calculate the solution $(\Delta U_{t+1}, \Delta V_{t+1})$ to the rank deficient least squares problem \eqref{eq:GNIMC_LSQR} whose norm $\|\Delta U_{t+1}\|_F^2 + \|\Delta V_{t+1}\|_F^2$ is minimal. The least squares problem \eqref{eq:GNIMC_LSQR} at iteration $t$ reads
\begin{align}\label{eq:GNIMC_LSQR_t}
\argmin_{\Delta U, \Delta V} \|\mathcal P_\Omega[A( U_tV_t^\top + U_t \Delta V^\top + \Delta U V_t^\top) B^\top] - Y \|_F^2 .
\end{align}
Denote the condition number of $X^*$ by $\kappa$. If $U_t, V_t$ are approximately balanced and their product $U_t V_t^\top$ is close to $X^*$, their condition number scales as $\sqrt\kappa$. Hence, the condition number of the least squares problem (namely, the condition number of the operator defined in \eqref{eq:LA} below) scales as $\sqrt \kappa$. As a result, directly solving \eqref{eq:GNIMC_LSQR_t} leads to a factor of $\sqrt\kappa$ in the computational complexity.
In the following, we describe a procedure that gives the same solution to \eqref{eq:GNIMC_LSQR_t} but eliminates the dependency in $\sqrt\kappa$, as proven in the next subsection.
The procedure consists of two phases. First, we efficiently compute a feasible solution to \eqref{eq:GNIMC_LSQR_t}, not necessarily the minimal norm one.
Second, we describe how, given a solution to \eqref{eq:GNIMC_LSQR_t}, we can efficiently compute the one with minimal norm, $(\Delta U_{t+1}, \Delta V_{t+1})$.
\cref{alg:efficient_procedure} provides a sketch of this procedure.\footnote{We remark that while the second phase works for any given feasible solution, in practice the feasible solution we find is also a minimal norm solution but of a different least squares problem. Since it also works well in practice, we did not need to employ the second phase in our simulations.}

By \cref{lem:full_rank}, the factor matrices of the current iterate $U_t, V_t$ are full column rank.
Let $Q_U R_U$ and $Q_V R_V$ be the QR decompositions of $U_t$ and $V_t$, respectively, such that $Q_U \in \mathbb R^{d_1\times r}$ and $Q_V \in \mathbb R^{d_2\times r}$ are isometries, and $R_U, R_V \in \mathbb R^{r\times r}$ are invertible.
Instead of \eqref{eq:GNIMC_LSQR_t}, we solve the following modified least squares problem,
\begin{align}\label{eq:GNIMC_modified_LSQR_t}
(\Delta U', \Delta V') &= \argmin_{\Delta U, \Delta V} \|\mathcal P_\Omega(AU_tV_t^\top B^\top + AQ_U \Delta V^\top B^\top + A\Delta U Q_V^\top B^\top) - Y \|_F^2.
\end{align}
Here, $(\Delta U', \Delta V')$ is any feasible solution to \eqref{eq:GNIMC_modified_LSQR_t}, not necessarily the minimal norm one.
Next, let
\begin{align}
\Delta U'' = \Delta U' (R_V^{-1})^\top
\quad \mbox{and} \quad
\Delta V'' = \Delta V' (R_U^{-1})^\top.
\end{align}
It is easy to verify that $(\Delta U'', \Delta V'')$ is a feasible solution to the original least squares problem \eqref{eq:GNIMC_LSQR_t}.
This concludes the first part of the procedure, which can be viewed as preconditioning: as we show below, \eqref{eq:GNIMC_modified_LSQR_t} has a lower condition number than \eqref{eq:GNIMC_LSQR_t}, and it hence faster to solve by iterative methods. The reason for the better conditioning is that $Q_U, Q_V$ both have condition number one rather than $\sqrt \kappa$.
The detailed computational complexity analysis is deferred to the next subsection.

Next, we describe how to transform a feasible solution, such as $(\Delta U'', \Delta V'')$, into the minimal norm one $(\Delta U_{t+1}, \Delta V_{t+1})$.
To this end, we first express the least squares operator in terms of the sensing operator $\mathcal A$ defined in \eqref{eq:sensingOperator_IMC}.
In the matrix sensing formulation, the least squares problem \eqref{eq:GNIMC_LSQR_t} reads
\begin{align*}
&\min_{(\Delta U, \Delta V)} \| \mathcal P_\Omega [A(U_t V_t^\top + U_t \Delta V^\top + \Delta U V_t^\top)B^\top] - Y ] \|_F \\
&= \min_{(\Delta U, \Delta V)} \| \text{Vec}_\Omega [A (U_t V_t + U_t \Delta V^\top + \Delta U V_t^\top)B^\top]/ \sqrt p - \text{Vec}_\Omega(Y)/\sqrt p \| \\
&= \min_{(\Delta U, \Delta V)} \|\mathcal A(U_t V_t^\top + U_t \Delta V^\top + \Delta U V_t^\top) - b\| \\
&= \min_{(\Delta U, \Delta V)} \|\mathcal A(U_t \Delta V^\top + \Delta U V_t^\top) - b_t\|,
\end{align*}
where $b = \text{Vec}_\Omega(Y)/\sqrt p$ and $b_t = b - \mathcal A(U_tV_t^\top)$.
The least squares operator $\mathcal L_{(U_t,V_t)}: \mathbb R^{(d_1+d_2)\times r}\to \mathbb R^{m}$ is thus
\begin{align}\label{eq:LA}
\mathcal L_{(U_t,V_t)} \begin{pmatrix} U \\ V \end{pmatrix} &= \mathcal A(U_t V^\top + U V_t^\top).
\end{align}
Let $\mathcal K = \ker \mathcal L_{(U_t,V_t)}$.
By combining our RIP guarantee (\cref{thm:IMC_RIP}) with the second part of Lemma~4.4 in \cite{zilber2022gnmr} which holds due to our \cref{lem:full_rank},
\begin{align}\label{eq:kernel}
\mathcal K = \left\{\begin{psmallmatrix} U_t R \\ -V_t R^\top\end{psmallmatrix} \,\mid\, R \in \mathbb R^{r\times r} \right\} = \left\{\begin{psmallmatrix} Q_U R \\ -Q_V R^\top\end{psmallmatrix} \,\mid\, R \in \mathbb R^{r\times r} \right\} .
\end{align}
Also, $\text{dim}\{\mathcal K\} = r^2$ as $Q_U, Q_V$ are isometries.
By definition of the minimal norm solution $\begin{psmallmatrix} \Delta U_{t+1} \\ \Delta V_{t+1} \end{psmallmatrix}$, any other solution is of the form $\begin{psmallmatrix} \Delta U'' \\ \Delta V''\end{psmallmatrix} = \begin{psmallmatrix} \Delta U_{t+1} \\ \Delta V_{t+1} \end{psmallmatrix} + \begin{psmallmatrix} K_U \\K_V \end{psmallmatrix}$ where $\begin{psmallmatrix} \Delta U_{t+1} \\ \Delta V_{t+1} \end{psmallmatrix} \perp \mathcal K$ and $\begin{psmallmatrix} K_U \\ K_V \end{psmallmatrix} \in \mathcal K$.
Hence, all we need to do is to subtract from $\begin{psmallmatrix} \Delta U'' \\ \Delta V''\end{psmallmatrix}$ its component in $\mathcal K$.
%
Denote the columns of $Q_U, Q_V$ by $u_i, v_i$ for $i\in [r]$, respectively, and let
\begin{align}\label{eq:kernel_basis_elements}
K^{(ij)} = \frac{1}{\sqrt 2} \begin{pmatrix} u_i e_j^\top \\ -v_j e_i^\top \end{pmatrix}, \quad \forall (i,j)\in [r]\times [r].
\end{align}
Then the following set of $r^2$ matrices form an orthonormal basis for the kernel $\mathcal K$ of \eqref{eq:kernel} under the Frobenius inner product $\braket{C,D} = \Tr(C^\top D)$:
\begin{align}\label{eq:kernel_basis}
\mathcal K_B = \left\{K^{(ij)} \,\mid\, (i,j)\in [r]\times [r] \right\} .
\end{align}
Let $\mathcal I$ be the identity operator.
By calculating the projector $\mathcal P_\mathcal K$ onto the span of $\mathcal K_B$, we obtain the minimal norm solution $\begin{psmallmatrix} \Delta U_{t+1} \\ \Delta V_{t+1} \end{psmallmatrix} = (\mathcal I - \mathcal P_\mathcal K) \begin{psmallmatrix} \Delta U'' \\ \Delta V'' \end{psmallmatrix}$.

The procedure described in this subsection is sketch in \cref{alg:efficient_procedure}.

\begin{algorithm}[t]
\caption{Efficient procedure to compute the minimal norm solution to \eqref{eq:GNIMC_LSQR}} \label{alg:efficient_procedure}
\SetKwInOut{Return}{return}
\SetKwInOut{Input}{input}
\SetKwInOut{Output}{output}
\SetKwComment{Comment}{$\triangleright$\ }{}
\Input{
sampling operator $\mathcal P_\Omega$, observed matrix $Y$, side information matrices $(A,B)$, current iterate $(U_t, V_t)$
}
\Output{the minimal norm solution to \eqref{eq:GNIMC_LSQR}}
\Comment{Phase I: compute a feasible solution to \eqref{eq:GNIMC_LSQR}}
compute $Q_U R_U$ and $Q_V R_V$, the QR decompositions of $U_t$ and $V_t$, respectively \label{alg:efficient_procedure_QR} \\
compute $(\Delta U', \Delta V')$, any feasible solution to
$ \argmin_{(\Delta U, \Delta V)} \| \mathcal P_\Omega[A(U_tV_t^\top + Q_U \Delta V^\top + \Delta U Q_V^\top)B^\top] - Y \|_F^2 $
\label{alg:efficient_procedure_lsqr} \\
set $\Delta U'' = \Delta U' (R_V^{-1})^\top$, $\Delta V'' = \Delta V' (R_U^{-1})^\top$ \label{alg:efficient_procedure_feasible} \\
\Comment{Phase II: compute the minimal norm solution to \eqref{eq:GNIMC_LSQR}}
let $\mathcal P_\mathcal K: \mathbb R^{(d_1+d_2)\times r} \to R^{(d_1+d_2)\times r}$ be the projector onto $\mathcal K$, using its orthonormal basis given in \eqref{eq:kernel_basis} \label{alg:efficient_procedure_orthognoalProjector} \\
set $\begin{psmallmatrix} \Delta U_{t+1} \\ \Delta V_{t+1} \end{psmallmatrix} = (\mathcal I - \mathcal P_\mathcal K) \begin{psmallmatrix} \Delta U'' \\ \Delta V'' \end{psmallmatrix}$ \label{alg:efficient_procedure_convert} \\
\Return{$(\Delta U_{t+1}, \Delta V_{t+1})$}
\end{algorithm}

\subsection{Proof of \cref{proposition:time_complexity}}
For the analysis of the computational complexity of \GNIMC with the minimal norm solution computed via \cref{alg:efficient_procedure}, we first prove the following auxiliary lemma.
Recall that the condition number of an operator $\mathcal P: \mathbb R^{(d_1+d_2)\times r}\to \mathbb R^{m}$ is defined as $\max_Z\{\|\mathcal P(Z)\|/\|Z\|_F\}/\min_Z\{\|\mathcal P(Z)\|/\|Z\|_F\}$.
\begin{lemma}\label{lem:LA_cond}
Let $\Omega, A, B$ be defined as in \cref{proposition:time_complexity}.
Let $\mathcal L_{(Q_U, Q_V)}$ be the least squares operator of step~\ref{alg:efficient_procedure_lsqr} in \cref{alg:efficient_procedure},
\begin{align}
\mathcal L_{(Q_U,Q_V)} \begin{pmatrix} U \\ V \end{pmatrix} &= \mathcal A(Q_U V^\top + U Q_V^\top).
\end{align}
Denote its condition number by $\kappa_L$. Then
\begin{align}\label{eq:LA_cond}
\kappa_L \leq \sqrt 6.
\end{align}
\end{lemma}

We remark that the bound in \eqref{eq:LA_cond} can be slightly improved (up to $\kappa_L \leq \sqrt 2$) at the cost of increasing $|\Omega|$.

\begin{proof}[Proof of \cref{lem:LA_cond}]
By combining assumption \eqref{eq:GNIMC_guarantee_sampleComplexity} and \cref{thm:IMC_RIP}, the sensing operator $\mathcal A$ satisfies a $\min\{d_1, d_2\}$-RIP with a constant $\delta \leq 1/2$. Hence, as in the proof of Lemma~4.2 in \cite{zilber2022gnmr}, the minimal nonzero singular value of $\mathcal L_{(Q_U,Q_V)}$, $\sigma_\text{min}(\mathcal L_{(Q_U,Q_V)})$, is bounded from below by $\sqrt{1-\delta} \min\{\sigma_r(Q_U), \sigma_r(Q_V)\}$ .
Next, by \cref{lem:full_rank}, $U_t$ and $V_t$ are of full column rank. Hence, $Q_U$ and $Q_V$ are isometries, and in particular $\sigma_r(Q_U) = \sigma_r(Q_V) = 1$. We thus obtain $\sigma_\text{min}(\mathcal L_{(Q_U,Q_V)}) \geq \sqrt{1-\delta}$.

We similarly bound from above the maximal singular value, $\sigma_1(\mathcal L_{(Q_U,Q_V)})$.
Let $U \in \mathbb R^{d_1\times r}$, $V \in \mathbb R^{d_2\times r}$. Then
\begin{align*}
\left\|\mathcal L_{(Q_U,Q_V)} \begin{pmatrix} U \\ V \end{pmatrix} \right\|^2 &\stackrel{(a)}{\leq} (1+\delta) \|UQ_V^\top + Q_UV^\top\|_F^2 \stackrel{(b)}\leq (1+\delta) (\|UQ_V^\top\|_F + \|Q_UV^\top\|_F)^2 \\
&\stackrel{(c)}= (1+\delta) (\|U\|_F^2 + 2 \|U\|_F\|V\|_F + \|V\|_F^2) \stackrel{(d)}{\leq} 2(1+\delta) (\|U\|_F^2 + \|V\|_F^2),
\end{align*}
where (a) follows by the RIP of $\mathcal A$, (b) by the triangle inequality, (c) by the fact that $Q_U, Q_V$ are isometries, and (d) by $ab \leq (a^2+b^2)/2$.
Hence, the maximal singular value is bounded from above by $\sqrt{2(1+\delta)}$.
The condition number of $\mathcal L_{(Q_U,Q_V)}$ is thus bounded as
\begin{align*}
\kappa_L \leq \sqrt\frac{2(1+\delta)}{1-\delta} \leq \sqrt{6}
\end{align*}
where the second inequality follows since $\delta \leq 1/2$.
\end{proof}

We are now ready to prove \cref{proposition:time_complexity}.
\begin{proof}[Proof of \cref{proposition:time_complexity}]
According to our quadratic convergence guarantee, the number of \GNIMC iterations till recovery with a fixed accuracy is constant. Thus, up to a multiplicative constant, the complexity of \GNIMC is the same as the complexity of a single iteration, which we shall now analyze according to its sketch in \cref{alg:efficient_procedure}.

The complexity of step~\ref{alg:efficient_procedure_QR}, which consists of QR factorizations of $d\times r$ matrices, is $\mathcal O(dr^2)$ \cite[Section~5.2.9]{golub2003separable}.
Step~\ref{alg:efficient_procedure_lsqr} is separately analyzed below.
Step~\ref{alg:efficient_procedure_feasible} is dominated by the calculation of the matrix product, which costs $\mathcal O(dr^2)$.
Steps \ref{alg:efficient_procedure_orthognoalProjector}-\ref{alg:efficient_procedure_convert} are dominated by the calculation of the projection of a feasible solution $\begin{psmallmatrix} \Delta U'' \\ \Delta V'' \end{psmallmatrix}$ onto the kernel $\mathcal K$ given in \eqref{eq:kernel}. To this end, we first construct a matrix $K \in \mathbb R^{(d_1+d_2)r\times r^2}$ whose columns are the vectorization of the elements of the orthonormal basis $\mathcal K_B$ given in \eqref{eq:kernel_basis}. Then, to obtain the required projection, we calculate the product $KK^\top z$ where $z \equiv \text{Vec }\begin{psmallmatrix} \Delta U'' \\ \Delta V'' \end{psmallmatrix} \in \mathbb R^{(d_1+d_2)r}$ is the vectorization of the feasible solution in hand.
By first calculating $K^\top z$ and then $K(K^\top z)$ we obtain the complexity of $\mathcal O(dr^3)$.

Finally, we analyze the complexity of step~\ref{alg:efficient_procedure_lsqr}. \GNIMC solves the least squares problem using the standard LSQR algorithm \cite{paige1982lsqr}, which applies the conjugate gradient (CG) method to the normal equations. Each inner iteration of CG is dominated by the calculation of $AQ_U\Delta V^\top B^\top + A\Delta U Q_V^\top B^\top$ at the entries of $\Omega$ \cite[Section~7.7]{paige1982lsqr}.
To obtain a single entry of $AQ_U\Delta V^\top B^\top$, we calculate a single row of $AQ_U$, a single column of $\Delta V^\top B^\top$, and then take the product. Since $Q_U \in \mathbb R^{d_1\times r}$ and $\Delta V^\top \in \mathbb R^{r\times d_2}$, this sums up to $\mathcal O(d_1r + d_2 r + r^2) \sim \mathcal O(dr)$ operations. Similarly, calculating a single entry of $A\Delta U Q_V^\top B^\top$ takes $\mathcal O(dr)$ operations. The complexity of a single iteration of CG is thus $\mathcal O(dr|\Omega|)$.

Next, we analyze the required number of CG iterations.
Let $\kappa_L$ be the condition number of the least squares operator $\mathcal L_{(Q_U, Q_V)}$ as defined in \cref{lem:LA_cond}.
The residual error of CG decays at least linearly with a contraction factor $\frac{\kappa_L-1}{\kappa_L+1}$ \cite[Section~4]{hayami2018convergence}.
By \cref{lem:LA_cond}, $\kappa_L$ is bounded by a constant, and hence the required number of CG iterations is also a constant.
We thus conclude that the total complexity of step~\ref{alg:efficient_procedure_lsqr} is $\mathcal O(dr|\Omega|)$.

Putting everything together, the complexity of \GNIMC is $\mathcal O(dr^3 + dr|\Omega|)$.
One of the conditions of the proposition is the lower bound $|\Omega| \geq 32 \mu^2 d_1d_2 \log n$.
W.l.o.g., we may assume $|\Omega| = 32 \mu^2 d_1d_2 \log n$ (if $|\Omega|$ is larger, we can ignore some of the observed entries). We thus obtain that the complexity of \GNIMC is $\mathcal O(\mu^2 d d_1 d_2 r \log n)$.

\end{proof}


\subsection{Comparison to gradient descent}
In this subsection we show that the per-iteration cost of \GNIMC, as analyzed above, is of the same order as that of gradient descent. Denote $E_\Omega = \mathcal P_\Omega(AUV^\top B^\top) - Y$, and let $f(U,V) = \|E_\Omega\|_F^2$ be the objective of the factorized matrix completion problem \eqref{eq:IMC_factorized}. Its gradient is
\begin{align*}
\nabla_U f(U,V) &= 2A^\top E_\Omega B V, \quad
\nabla_V f(U,V) = 2B^\top E_\Omega^\top A U .
\end{align*}
As explained in the analysis of step 2 above, calculating $E_\Omega$ costs $\mathcal O(dr|\Omega|)$. Since $A^\top E_\Omega$ and $B^\top E_\Omega^\top$ have at most $r|\Omega|$ nonzero entries, this is also the cost of calculating $(A^\top E_\Omega) B$ and $(B^\top E_\Omega^\top) A$. Finally, calculating $(A^\top E_\Omega B) V$ and $(B^\top E_\Omega^\top A) U$ is $\mathcal O(d_1d_2 r)$. The per-iteration complexity of gradient descent is thus $\mathcal O(dr|\Omega| + d_1d_2r)$. Under assumption \eqref{eq:GNIMC_guarantee_sampleComplexity} on $|\Omega|$, this coincides with the per-iteration complexity of \GNIMC.

We remark that empirically, the overall complexity (namely, from initialization to convergence) of gradient descent seems so be much larger than that of \GNIMC, see \cref{sec:experiments}. This observation is in agreement with the theoretical analysis of gradient descent in \cite{zhang2018fast}, see their Theorem~5.5.

\section{Proof of \cref{proposition:initialization} (Initialization Guarantee)} \label{sec:proof_initialization}
By its construction, $\begin{psmallmatrix} U_0 \\ V_0 \end{psmallmatrix}$ is perfectly balanced, $U_0^\top U_0 = V_0^\top V_0$, and thus satisfies \eqref{eq:initialization_balance}. Hence, we only need to prove \eqref{eq:initialization_accuracy}.
Let $\mathcal A$ be the sensing operator that corresponds to \eqref{eq:IMC} as defined in \eqref{eq:sensingOperator_IMC}.
By \cref{thm:IMC_RIP}, w.p.~at least $1-2n^{-2}$, the operator $\mathcal A$ satisfies a $\min\{d_1,d_2\}$-RIP with a constant $\delta = 2/5$.
Hence, according to Lemmas~5.1-5.2 in \cite{tu2016low}, or more explicitly Eq.~(5.26) in their extended arXiv version, after $\tau \geq \log (c\sqrt r\kappa) / \log(1/(2\delta)) \geq 5\log (c\sqrt r\kappa)$ iterations of \eqref{eq:initialization_proceudre} we have $\|M_\tau - M^*\|_F \leq \sigma_r^*/c$. Since $A,B$ are isometries, $\|AM_\tau B^\top - X^*\|_F = \|M_\tau - M^*\|_F \leq \sigma_r^*/c$.
Hence $\begin{psmallmatrix} U_0 \\ V_0 \end{psmallmatrix}$ satisfies \eqref{eq:initialization_accuracy} for any $\tau \geq 5\log (c\sqrt r\kappa)$.
\qed

\section{Comparison to \cite{ghassemi2018global}}\label{sec:comparisonToGhassemi2018}
\cite{ghassemi2018global} derived results analogous to our \cref{thm:IMC_RIP,thm:IMC_landscape}. However, there are three main differences between the claims.
First, the sample complexity for the RIP result in \cite{ghassemi2018global} is
\begin{align}\label{eq:Ghassemi_sampleComplexity}
\mathcal O(\mu^2 r \max\{d_1,d_2\} \max\{d_1d_2, \log^2 n\} \log(1/\delta) /\delta^2),
\end{align}
compared to our $\mathcal O(\mu^2 d_1d_2 \log(n) / \delta^2)$.
We remark that in their notation, their claimed sample complexity is $\mathcal O(\mu^2 \max\{d_1,d_2\} \bar r^2 r)$ where $\bar r = \max\{r, \log n\}$. However, there seems to be an error in their analysis. Their assumption is that $A/\sqrt{n_1}$ and $B/\sqrt{n_2}$ are isometries, and their corresponding incoherence assumption is $\|A\|_{2,\infty}^2 \leq \mu \bar r$ and $\|B\|_{2,\infty}^2 \leq \mu \bar r$ with a constant $\mu$ \cite[Assumption~1]{ghassemi2018global}. But since $A \in \mathbb R^{n_1\times d_1}$ and $B\in \mathbb R^{n_2\times d_2}$, either the assumption should be $\|A\|_{2,\infty}^2 \leq \mu d_1$ and $\|B\|_{2,\infty}^2 \leq \mu d_2$, or the parameter $\mu$ is not a constant but rather scales with $\max\{d_1,d_2\}/\bar r$. In any case, in our notation their sample complexity is as given in \eqref{eq:Ghassemi_sampleComplexity}.

Second, the RIP result in \cite{ghassemi2018global} is for rank-$\min\{2r, d_1, d_2\}$ matrices, compared to our stronger rank-$\min\{d_1,d_2\}$ RIP. We remark that while their RIP result is formulated as $2r$-RIP, this implicitly assumes $r \ll \min\{d_1,d_2\}$ (see also the discussion in \cref{sec:proof_RIP_consequences}). In the general case, their guarantee is $\min\{2r, d_1, d_2\}$-RIP.

Third, \cite{ghassemi2018global} prove benign optimization landscape for the problem
\begin{align}\label{eq:Ghassemi2018_regularization}
\min_{U,V} \| \mathcal P_\Omega(AUV^\top B^\top) - Y \| + \frac 14 \|UU^\top - VV^\top \|_F^2,
\end{align}
which is an imbalance regularized version of \eqref{eq:IMC_factorized}.
Furthermore, it seems that their result cannot be readily extended to the vanilla IMC problem, as the regularization in \eqref{eq:Ghassemi2018_regularization} is essential in their proof.

\section{Additional Simulation Details}\label{sec:additional_experimental_details}
All algorithms are initialized with the same procedure, which is the spectral initialization, except for \texttt{Maxide} which is not factorization based and is by default initialized with the zero matrix.

\texttt{Maxide} gets as input a regularization parameter, and \texttt{MPPF}, \texttt{GD}, \texttt{RGD} and \texttt{ScaledGD} get a step size parameter.
For each simulation setting, we tuned the optimal parameter out of $10$ logarithmically-scaled values. The permitted values for \texttt{Maxide} were $10^{-5}, ..., 10^{-14}$, for \texttt{MPPF}, \texttt{GD} and \texttt{RGD} were $10^{-2} / \kappa, ..., 10^{\frac 12} / \kappa$ where $\kappa$ is the condition number of $X^*$, and for \texttt{ScaledGD} were $10^{-2}, ..., 10^{\frac 12}$. In all simulations, we verified that the selected value is an interior point of the permitted set, so that it is close to optimal. We remark that the regularization coefficient of \texttt{MPPF} and \texttt{RGD} can also be tuned, but we observed it has a very little effect. 
For \GNIMC, in all simulations we identically set the maximal number of iterations for the inner least-squares solver to $10$ if the observed error is low, $\frac{\|\mathcal P_\Omega (X_t) - Y\|_F}{\|Y\|_F} \leq 10^{-4}$, and $1000$ otherwise. This scheme exhibits slightly better performance than setting a constant value of maximal inner iterations (but only marginally). While tuning this value for each simulation independently, as we did for the hyperparameters of the above algorithms, may enhance performance, we preferred to demonstrate the performance of a tuning-free version of \GNIMC.

We used the following two stopping criteria for all methods: (i) small relative observed RMSE, $\frac{\|\mathcal P_\Omega (X_t) - Y\|_F}{\|Y\|_F} \leq \epsilon$, or (ii) small relative estimate change $\frac{\|\mathcal P_\Omega (X_t - X_{t-1})\|_F}{\|\mathcal P_\Omega(X_t)\|_F} \leq \epsilon$.
In our simulations, we set $\epsilon = 10^{-14}$.
For a fair comparison, we disabled all the other early stopping criteria defined in the algorithms.

\section{Additional Simulation results}\label{sec:additional_experimental_results}
In subsection \ref{sec:additional_experiments_noise} we demonstrate the stability of \GNIMC to Gaussian noise. In subsection \ref{sec:additional_experiments_conditionNumber} we demonstrate the insensitivity of several algorithms to the condition number of the underlying matrix in terms of the number of observations required for recovery.

\subsection{Stability of \GNIMC to noise}\label{sec:additional_experiments_noise}

\begin{figure}[t]
\centering
\includegraphics[width=0.48\linewidth]{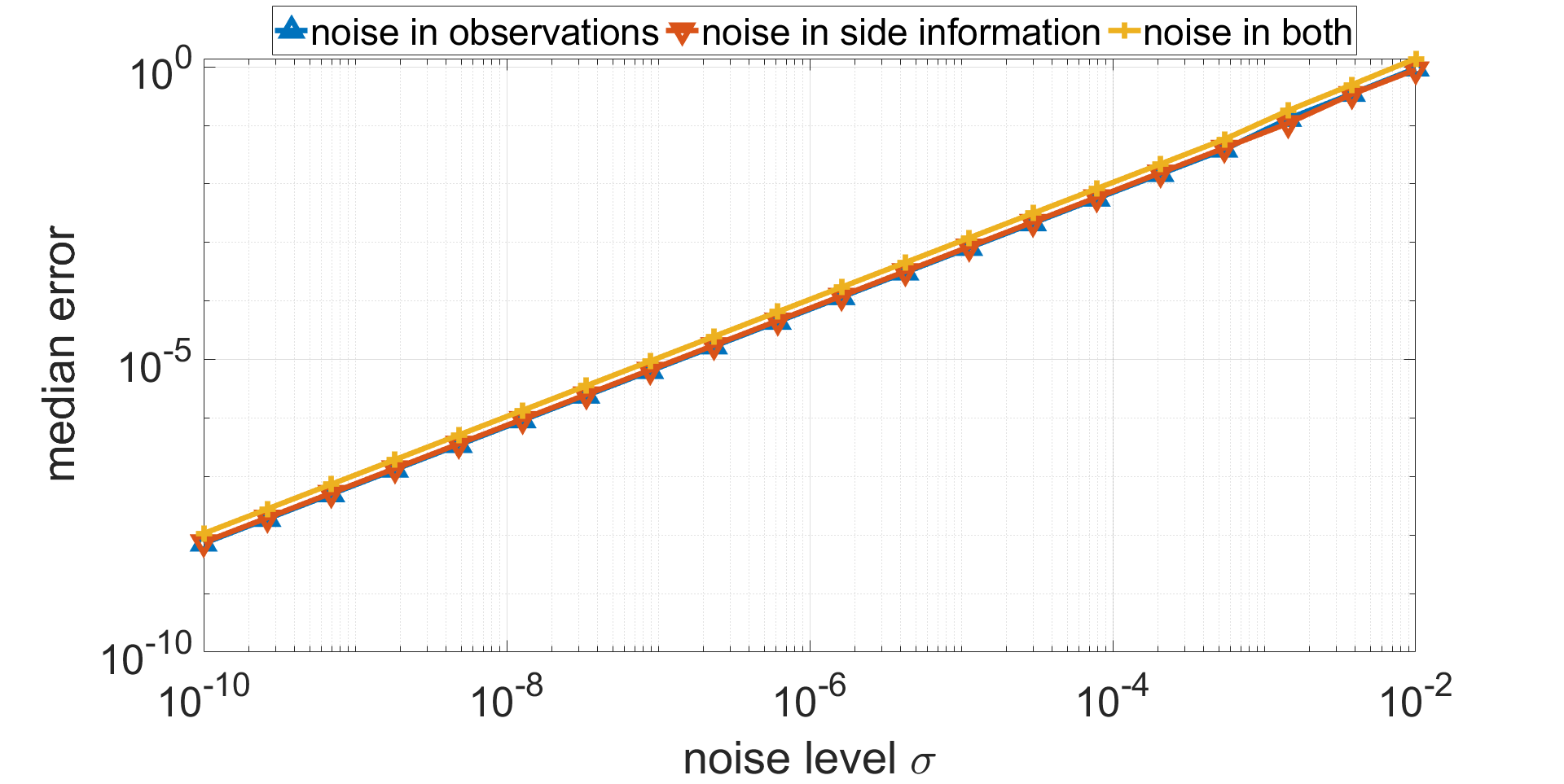}
 \caption{Stability of \GNIMC to additive Gaussian noise, with the same settings as in \cref{fig:convergence_recoveryVsCN}(left).}
\label{fig:noise}
\end{figure}

\Cref{fig:noise} demonstrates the stability of \GNIMC to noise.
In this simulation, either the observed entries $Y$, the side information matrices $A,B$, or both, are corrupted by additive Gaussian noise of zero mean and standard deviation $\sigma$.
As seen in the figure, the error of \GNIMC scales linearly with the noise level $\sigma$.

\subsection{Insensitivity of several algorithms to the condition number}\label{sec:additional_experiments_conditionNumber}
In \cref{fig:convergence_recoveryVsCN}(right) we addressed the sensitivity (or insensitivity) of several IMC algorithms to the condition number of $X^*$ in terms of their runtime.
In this subsection, we explore another aspect of sensitivity to the condition number: rather than runtime, we study how the condition number affects the number of observations required for a successful recovery given no time constraints.

\renewcommand{\arraystretch}{1.25}
\begin{table}[t]
\caption{Lowest oversampling ratio $\rho$ from which the median of \texttt{rel-RMSE} \eqref{eq:relRMSE} is lower than $10^{-4}$, as a function of the condition number $\kappa$, in the setting $n_1 = n_2 = 1000$, $d_1 = d_2 = 20$, $r = 10$. Each entry shows the median of $50$ independent realizations.}
\label{table:oversamplingVsCN}
\begin{center}
\begin{tabularx}{0.45\textwidth}{c X X X X X}
 \hline
 \diagbox[height=1.5\line]{Alg.}{$\kappa$} & \,\,$1$ & $10$ & $10^2$ & $10^3$ & $10^4$ \\
 \hline
 \texttt{GNIMC} & $1.2$ & $1.1$ & $1.1$ & $1.1$ & $1.1$ \\
 \texttt{AltMin} & $1.1$ & $1.1$ & $1.1$ & $1.1$ & $1.1$ \\
 \texttt{GD} & $1.1$ & $1.1$ & $1.1$ & $1.1$ & $1.1$ \\
 \texttt{RGD} & $1.2$ & $1.2$ & $1.1$ & $1.1$ & $1.1$ \\
  \texttt{ScaledGD} & $1.2$ & $1.1$ & $1.2$ & $1.2$ & $1.1$ \\
 \texttt{Maxide} & $1.4$ & $1.4$ & $1.4$ & $1.4$ & $1.4$ \\
 \hline
\end{tabularx}
\end{center}
\end{table}

In our simulations, we observed the following interesting phenomenon:
For all algorithms, the number of observations $|\Omega|$ required for recovery is independent of the condition number $\kappa$.
We demonstrate this in \cref{table:oversamplingVsCN}, which compares the minimal oversampling ratio, out of the values $\rho = 1.1, 1.2, ...$, required by several algorithms to reach relative RMSE of $10^{-4}$.
Since in this experiment our goal is to explore fundamental recovery abilities rather than speedy performance, the algorithms are given essentially unlimited runtime (in practice, the time limit was set to one CPU hour, and $3$ hours for \texttt{GD} and \texttt{RGD} in the case of $\kappa=10^4$).
The table shows that the minimal oversampling ratio does not increase with $\kappa$; in fact, it sometimes slightly decreases when $\kappa$ is small. We did not include \texttt{MPPF} in \cref{table:oversamplingVsCN} due to its long runtime; however, a limited set of simulations suggests that the same conclusion also holds for it.

Beyond illustrating the abilities of the algorithms, this result demonstrates a basic property of the IMC problem: insensitivity to the condition number.
This result corresponds well with our RIP guarantee for IMC, \cref{thm:IMC_RIP}, as the RIP holds for all matrices of certain ranks regardless of their condition number.


\end{document}